\documentclass[twoside,11pt]{article}

\usepackage[abbrvbib, preprint]{jmlr2e}

\usepackage[american]{babel}
\usepackage[utf8]{inputenc} 
\usepackage[T1]{fontenc}    
\usepackage{microtype}
\usepackage{graphicx}
\usepackage{booktabs} 
\usepackage{natbib}
\usepackage{hyperref}       
\usepackage{url}            
\usepackage{booktabs}       
\usepackage{amsfonts}       
\usepackage{nicefrac}       
\usepackage{microtype}      
\usepackage{amssymb}
\usepackage{amsmath}
\usepackage{mathtools}
\usepackage{graphicx}
\usepackage{subcaption}
\usepackage{float}
\usepackage{placeins}
\usepackage{xcolor}
\usepackage[ruled,vlined]{algorithm2e}
\usepackage{xr}
\usepackage{minitoc}
\usepackage[shortlabels]{enumitem}

\usepackage{thmtools,thm-restate}

\newcommand*\samethanks[1][\value{footnote}]{\footnotemark[#1]}

\newcommand{\norm}[1]{\left\lVert#1\right\rVert}

\addto\captionsamerican{%
  \renewcommand{\partname}%
    {}%
}    
\addto\captionsamerican{%
  \renewcommand{\thepart}%
    {}%
}

\jmlrheading{1}{2000}{1-48}{4/00}{10/00}{meila00a}{Sujay Thakur, Cooper Lorsung, Yaniv Yacoby, Finale Doshi-Velez and Weiwei Pan}

\ShortHeadings{Uncertainty-Aware Bases for Deep Bayesian Regression}{Thakur, Lorsung, Yacoby, Doshi-Velez \& Pan}
\firstpageno{1}

\begin{document}

\title{Uncertainty-Aware (UNA) Bases for Deep Bayesian Regression using Multi-Headed Auxiliary Networks}

\author{\name Sujay Thakur\thanks{Equal contribution} \email sujay\_thakur@g.harvard.edu \\
    \name Cooper Lorsung\samethanks \email clorsung@g.harvard.edu \\
    \name Yaniv Yacoby\samethanks \email yanivyacoby@g.harvard.edu \\
	\name Finale Doshi-Velez \email finale@seas.harvard.edu \\
	\name Weiwei Pan \email weiweipan@g.harvard.edu \\
       \addr John A. Paulson School of Engineering and Applied Sciences\\
       Harvard University\\
       Cambridge, MA 02138, USA
}

\editor{Kevin Murphy and Bernhard Sch{\"o}lkopf}

\maketitle

\doparttoc 
\faketableofcontents 

\begin{abstract}
Neural Linear Models (NLM) are deep Bayesian models that produce predictive uncertainties by learning features from the data and then performing Bayesian linear regression over these features. 
Despite their popularity, few works have focused on methodically evaluating the predictive uncertainties of these models.
In this work, we demonstrate that traditional training procedures for NLMs drastically underestimate uncertainty on out-of-distribution inputs, and that they therefore cannot be naively deployed in risk-sensitive applications. 
We identify the underlying reasons for this behavior and propose a novel training framework that captures useful predictive uncertainties for downstream tasks.
\end{abstract}


\section{Introduction}

In high-stakes, safety-critical applications of machine learning, reliable measurements of model predictive uncertainty matter just as much as predictive accuracy. Traditionally, applications requiring predictive uncertainty relied on Gaussian Processes (GPs) \citep{rasmussen_gp} for two reasons: (1) with an appropriately chosen kernel, they produce high predictive uncertainty on out-of-distribution inputs and low uncertainty on within-distribution inputs, and (2) their predictive uncertainty can be easily and meaningfully tuned via a small set of hyper-parameters of the kernel function (for example, the smoothness of GPs with an RBF kernel can be tuned via the length-scale and amplitude), allowing domain experts to encode task-relevant knowledge.
However, the computational complexity of GP inference has motivated research on fast and scalable alternatives -- specifically, deep Bayesian models with approximate inference (e.g.  \cite{springenberg_bayesian,Snoek}). 

Bayesian Neural Networks (BNNs) \citep{neal_bnn}, for example, provide a way of explicitly capturing model uncertainty -- uncertainty arising from having insufficient observations to determine the ``true'' predictor -- by placing a prior distribution over network weights. 
Just like GP inference (rather than point estimates), Bayesian inference for BNNs produces a distribution over possible predictions, whose variance can be used as an indicator of model uncertainty during test time. 
While promising, these alternatives unfortunately do not retain the two desired properties of GPs. That is, approximate inference methods yield models that are overly certain on test points coming from data-poor regions of the input space \citep{uncertainty_quality,uci_gap}, and require an unintuitive hyper-parameter tuning process in order to achieve task-appropriate behavior~\citep{sun2018functional}. 
Furthermore, unlike in the case of GP models, it is much more difficult to encode domain knowledge or functional knowledge (prior knowledge about the true predictor) in deep Bayesian models \citep{sun2018functional}, and hence the predictive uncertainties of these models are often difficult to tailor to specific downstream tasks. 

For this reason, the Neural Linear Model (NLM), a model with BNN-like properties but highly tractable inference, is gaining popularity ~\citep{Snoek, rl_nlm, activelearning_nlm, automl_nlm}.
An NLM places a prior only on the last layer of network weights and learns point estimates for the remaining weights; inference for the last weight-layer can then be performed analytically. 
One can interpret the deterministic network layers as a finite-dimensional feature-space embedding of the data, and the last layer of NLMs as performing Bayesian linear regression on the \emph{feature basis}, that is, the basis defined by the feature embedding of the data. 

Although NLMs are easy to implement and scale to large data-sets, in order to safely deploy them in risk-averse applications, we nonetheless need to verify that these models retain desirable properties of GPs. But despite their increasing popularity, little work has been done to methodically evaluate the quality of uncertainty estimates produced by NLMs. 
In the first paper to do so \citep{Rasmussen}, the authors show that NLMs can achieve high log-likelihood on test data sampled from training-data-scarce regions; they treat this as evidence that NLM uncertainties can distinguish data-scarce and data-rich regions. 
However, as noted by \cite{uncertainty_quality}, log-likelihood measures only how well predictive uncertainty aligns with the variation in the observed data and not how well these uncertainties predict data scarcity. In fact, in this work, we will show that the predictive uncertainties of NLMs resulting from traditional inference are overly confident on training-data-poor regions. 

In this paper, we describe a novel NLM training framework, UNcertainty-Aware (UNA) training, for producing predictive uncertainties that can distinguish data-rich from data-poor regions. UNA training retains the speed and scalability of traditional NLM inference while explicitly encouraging desirable GP-like properties in the learned model. Our contributions are both theoretical and methodological:

\textbf{1. We demonstrate that all three traditional training objectives for NLMs -- MLE, MAP, and maximum marginal likelihood -- all yield predictive uncertainty that cannot distinguish data-scarce from data-rich regions.} 
Furthermore, we identify the precise cause of the problem -- traditional NLM training procedures learn feature bases incapable of expressing uncertainty in data-scarce regions (also known as ``in-between" uncertainties ~\citep{uci_gap}).

\textbf{2. We propose a new framework, UNA training, for learning uncertainty-aware and task-aware feature bases for NLMs.} Our framework trains a set of auxiliary regressors on a shared feature basis. By specifying the properties needed of these auxiliary regressors for good downstream performance into a training objective, we explicitly encourage the learned feature basis to satisfy task-specific desiderata.
Our framework is both scalable and easy to implement. 

\textbf{3. We propose an instantiation of this framework, LUNA, for tasks requiring models that express in-between uncertainty.}
That is, LUNA is designed for any task that requires models to be more uncertain where there is little data relative to where there is data.
We do this by training an auxiliary set of linear predictors on top of the NLM's feature basis to extrapolate differently outside the data-rich regions (thus producing larger predictive variance in these regions). 

\textbf{4. We empirically demonstrate the utility of LUNA training on a wide range of downstream tasks.} (a) On a number of synthetic and real datasets, models trained with LUNA reliably identify data-scarce regions where baselines, including NLM with traditional training, struggle; (b) on transfer learning tasks, we are able to learn bases that outperform bases from traditional NLM training; (c) on Bayesian Optimization benchmarks, models learned with LUNA are comparable to baselines.

\section{Related Works} 
\label{sec:related_works}

\paragraph{Gaussian Processes.} While expressive and intuitive to tune, inference for Gaussian Processes (GPs) is computationally challenging for large datasets, scaling cubically with respect to the total number of observations. A large body of the literature about increasing the computational efficiency of GP inference focuses on approximate methods like inducing points (e.g. \cite{snelson2006sparse}), random feature expansions (e.g. \cite{wilson2016deep}) or stochastic variational optimization (e.g. \cite{cheng2017variational}). However, it is not known how well these methods approximate the performance of GP models with exact inference. On the other hand, the recent breakthrough in fast exact GP inference leverages multi-GPU parallelization and may be inappropriate when computational resources are limited \citep{wang2019exact}. In this work, we focus on training NLMs with modest-sized neural network architectures that nonetheless retain desired properties of GPs.
Lastly, because GPs are non-parametric, it is non-trivial for them to incorporate recent advances in deep learning architectures (like convolutional, recurrent, or graph structures) without significant computational overhead, requiring novel inference methods (e.g. ~\cite{ConvGP2017,Mattos2019,walker19a}).
In contrast, in this work we focus on NLMs, which can trivially be adapted to incorporate new innovations in neural network architectures. 

\paragraph{Bayesian Neural Networks.}
Early work on Bayesian Neural Networks (BNN) inference focuses on Hamiltonian Monte Carlo (HMC) \citep{neal2012bayesian} and Laplace approximations of the posterior \citep{mackay1992practical, buntine1991bayesian}. While HMC remains the ``gold standard'' for BNN inference, it does not scale well to large architectures or datasets; classical Laplace approximation, like Linearised Laplace \citep{uci_gap}, has similar difficulties scaling to modern architectures with large parameter sets. Although variational inference methods can be easily applied to BNN models with larger architectures, a number of these methods, like mean-field variational inference (VI) \citep{anderson1987mean, hinton1993keeping, blundell2015weight} and Monte Carlo Dropout (MCD) \citep{mcd} (which can be recast as a form of approximate variational inference with a spike and slab variational distribution), have recently been shown to underestimate predictive uncertainty, especially in data-scarce regions \citep{uci_gap, uncertainty_quality, mcd_pathologies}. 

\paragraph{Bayesian Models with Deterministic Neural Network Feature Extractors.} In order to bypass the difficulties of Bayesian inference for complex models, a number of works divide the models into two parts -- a neural network, trained deterministically, and a simple Bayesian model, for which inference can be performed exactly and/or scalably. These works can broadly be divided into two categories: ones for which the simple Bayesian model is a GP, and ones for which it is a Bayesian linear (or logistic) regression. 
Manifold Gaussian Processes (MGPs), for example, jointly train a GP on top of a neural network feature extractor~\citep{calandra2016manifold}. They are made scalable by \cite{liu2020simple}, who use a Random Fourier Feature (RFF) expansion GP approximation and isometry-enforcing regularization on the neural network feature map (SNGP). However, since high-quality inference requires a large number of features in the RFF expansion~\citep{random_bases}, inference for these models is still computationally expensive. In contrast to MGPs, the Neural Linear Model (NLM) \citep{Snoek} jointly trains a Bayesian linear regression (instead of a GP) on top of a deterministic neural network feature extractor.
Inference for NLMs therefore does not scale cubically with the number of observations (as in GP inference), and does not scale cubically with the number of features in the RFF expansion (as in SNGP); instead, it scales cubically with the number of features extracted by the deterministic neural network.
So long as the number of features extracted by the deterministic neural network remains relatively small NLMs are cheap to train.
While NLMs have been successfully applied in a number of applications requiring predictive uncertainty like Bayesian Optimization (BayesOpt) \citep{Snoek}, in this paper, we show that traditional inference for NLMs will, in most cases, underestimate uncertainty in data-scarce regions. Specifically, we show that the (relatively small number of) features extracted by the neural network hinder the Bayesian model from capturing in-between uncertainty. We then propose a novel training method that alleviates this issue.  

\paragraph{Ensemble Methods.} Alternatively, one can avoid Bayesian inference all together by ensembling (non-Bayesian) neural networks in order to estimate predictive uncertainty using the variance of predictions in the ensemble. For the variance of the predictions in the ensemble to be higher in data-sparse regions of the input space, the ensemble members must be diverse. Some works rely on bootstrapping to achieve this, or on multiple random restarts and adversarial training in ensemble building (e.g. \cite{lakshminarayanan2017simple}). Others, like the work of \cite{pearce2018bayesian}, relate ensembling to approximate Bayesian inference -- i.e. randomized MAP sampling \citep{lu2017ensemble, garipov2018loss} -- with the introduction of noise in the regularization term of each network (which in turn encourages for functional diversity). Though our focus in this paper is on Bayesian models, we nonetheless compare our methods to ensemble baselines.

\section{Background}
\label{sec:background}

Let the input space be $D$-dimensional, and suppose we have a dataset $\mathcal{D} = \{(\mathbf{x}_1, y_1), \ldots,  (\mathbf{x}_N, y_N)\}$ of $N$ observations, where $\mathbf{x}_n \in \mathbb{R}^D$ and $\mathbf{y}_n \in \mathbb{R}$. A Neural Linear Model (NLM) consists of: (1) a feature map $\phi_\theta:\mathbb{R}^D \to \mathbb{R}^L$, parameterized by a neural network with weights $\theta$, and (2) a Bayesian linear regression model fitted on the data embedded in the feature space:
\begin{equation*}
\mathbf{y} \sim \mathcal{N}(\mathbf{\Phi_\theta w}, \sigma^2\mathbf{I}), \quad 
\mathbf{w} \sim \mathcal{N}(\mathbf{0}, \alpha\mathbf{I})
\end{equation*}
where the design matrix 
$
\mathbf{\Phi_\theta} = [\widetilde{\phi_\theta(\mathbf{x}_1)}, \hdots, \widetilde{\phi_\theta(\mathbf{x}_N)]^\intercal}
$ is called the \emph{feature basis} and $\widetilde{\phi_\theta(\mathbf{x}_n)}$
is the feature vector $\phi_\theta(\mathbf{x}_n)$ augmented with a 1 (for a bias term). Thus, given the learned feature map, the NLM's posterior, marginal likelihood, and posterior predictive distributions can all be computed analytically. Intuitively, an NLM represents a neural network with a Gaussian prior over the last-layer weights $\mathbf{w}$, and with deterministic weights $\theta$ for the remaining layers. 

Inference for NLMs consists of two steps: 
\begin{itemize}
    \item[] \textbf{Step I:} Learn $\theta$ (with some objective, described below).
    \item[] \textbf{Step II:} Given $\theta$, infer $p(\mathbf{w} | \mathcal{D}, \theta)$ analytically.
\end{itemize}
In Step I, there are three accepted methods of learning $\theta$: maximum likelihood (MLE), maximum a posteriori (MAP) and Marginal-Likelihood, of which MAP is the most common. 
In MAP training~\citep{Snoek}, one maximizes the likelihood of the observed data with respect to $\theta$ and with respect to a point estimate, $\mathbf{\widetilde{w}}$, of the last layer's weights (i.e. we train the entire network deterministically, with an $\ell_2$-regularization term on the weights of the entire network):
\begin{align}\label{eq:nlm_obj}
\mathcal{L}_{\mathrm{MAP}}(\theta_\mathrm{Full}) = \log{ \mathcal{N}\left(\mathbf{y}; \mathbf{\Phi_\theta w}, \sigma^2\mathbf{I}\right)} - \gamma\norm{\theta_\mathrm{Full}}_2^2,
\end{align}
where $\theta_\mathrm{Full} = (\theta, \mathbf{\widetilde{w}})$ are weights of the full network. 
In Step II, we \emph{discard} $\mathbf{\widetilde{w}}$ and use the $\theta$ learned in Step I to infer $p(\mathbf{w} | \mathcal{D}, \theta)$ analytically.
MLE training is the same as MAP training, but with $\gamma = 0$. 
Lastly, in Marginal-Likelihood training, Step I consists of maximizing the likelihood of the observed data, after having marginalized out the weights $\mathbf{w}$:
\small
\begin{align}\label{eq:marginal_ll}
\begin{split}
\mathcal{L}_{\mathrm{Marginal}}(\theta) = \log{\mathbb{E}_{p(\mathbf{w})} \left\lbrack \mathcal{N}\left(\mathbf{y}; \mathbf{\Phi_\theta w}, \sigma^2\mathbf{I}\right) \right\rbrack} - \gamma\norm{\theta}_2^2.
\end{split}
\end{align}
\normalsize

In this paper, we show that all three inference methods for learning the feature basis (determined by $\theta$) in Step I produce models that are unable to distinguish between data-poor and data-rich regions (i.e. these models fail to capture in-between uncertainty).
In Section \ref{sec:method} we then propose a novel framework for training NLMs that learns models capable of expressing in-between uncertainty.

\section{Analysis of the Expressiveness of Neural Linear Model Uncertainties}
\label{sec:nlm_pathologies}

In this section, we show that conventional NLM training objectives result in models with predictive uncertainties that fail to distinguish data-rich from data-poor regions. Moreover, we identify the precise cause of the problem: none of the methods encourage diversity in the class of functions that are likely under the prior predictive; in fact, some training objectives explicitly discourage diversity. As a result, the posterior predictive of the learned model will be distributed over a limited function class. \emph{The limited functional variation across the input domain under the posterior predictive causes underestimation of in-between uncertainty}. 

\paragraph{Failure of MAP Training.}
When the regularization term in the MAP objective (Equation \ref{eq:nlm_obj}) is non-zero, the feature map $\phi_\theta$ is explicitly discouraged from producing bases spanning functions that extrapolate differently away from the observed data.
This is because such diversity comes at the cost of larger values in $\theta$ and does not impact the log-likelihood of the observed data. 

In Figure \ref{fig:nlm_priors}, we show samples from the prior predictive for two NLMs -- with a regularized and unregularized feature map,  $\phi_\theta$, respectively. With the regularization ($\gamma = 10.0$), the feature basis spans a limited set of functions -- the prior predictive samples show \emph{no variation} in the data-scarce region.
So what kind of posterior predictives does this prior predictive induce? Figure \ref{fig:nlm_posteriors} shows that when the prior predictive is inexpressive, the posterior predictive shows little in-between uncertainty. 
In Figure \ref{fig:rr_reg} of Appendix \ref{sec:appendix_toy_results}, we reproduce the effect of regularization on NLM prior predictives for different values of $\gamma$ and over random restarts, and show that (a) regularization consistently leads to inexpressive prior predictives, and (b) that as a result, the posterior predictives are nearly as certain in data-poor regions as they are in data-rich regions. 

So what happens when we train the NLM without regularization ($\gamma = 0.0$)? In this case, the NLM is not explicitly discouraged from expressing diverse functions. As such, for \emph{for the particular random restart in Figure \ref{fig:nlm_map}}, the basis for $\phi_\theta$ spans a diverse class of functions under the prior distribution;
that is, linear combinations of the features $\phi_\theta$ under $p(\mathbf{w})$ show variation in both the data-rich and data-scarce regions. 
Correspondingly, Figure \ref{fig:nlm_posteriors} shows that the posterior predictive is expressive. 
Based solely on these results, one might suppose that traditional training for NLMs without regularization (i.e. MLE training) does not suffer from the aforementioned issues. 
However, as we discuss next, MLE training cannot consistently learn models that capture in-between uncertainty. 

\begin{figure}[p]
    \centering
    \begin{subfigure}{1.0\linewidth}
    \centering
    \includegraphics[width=0.85\linewidth]{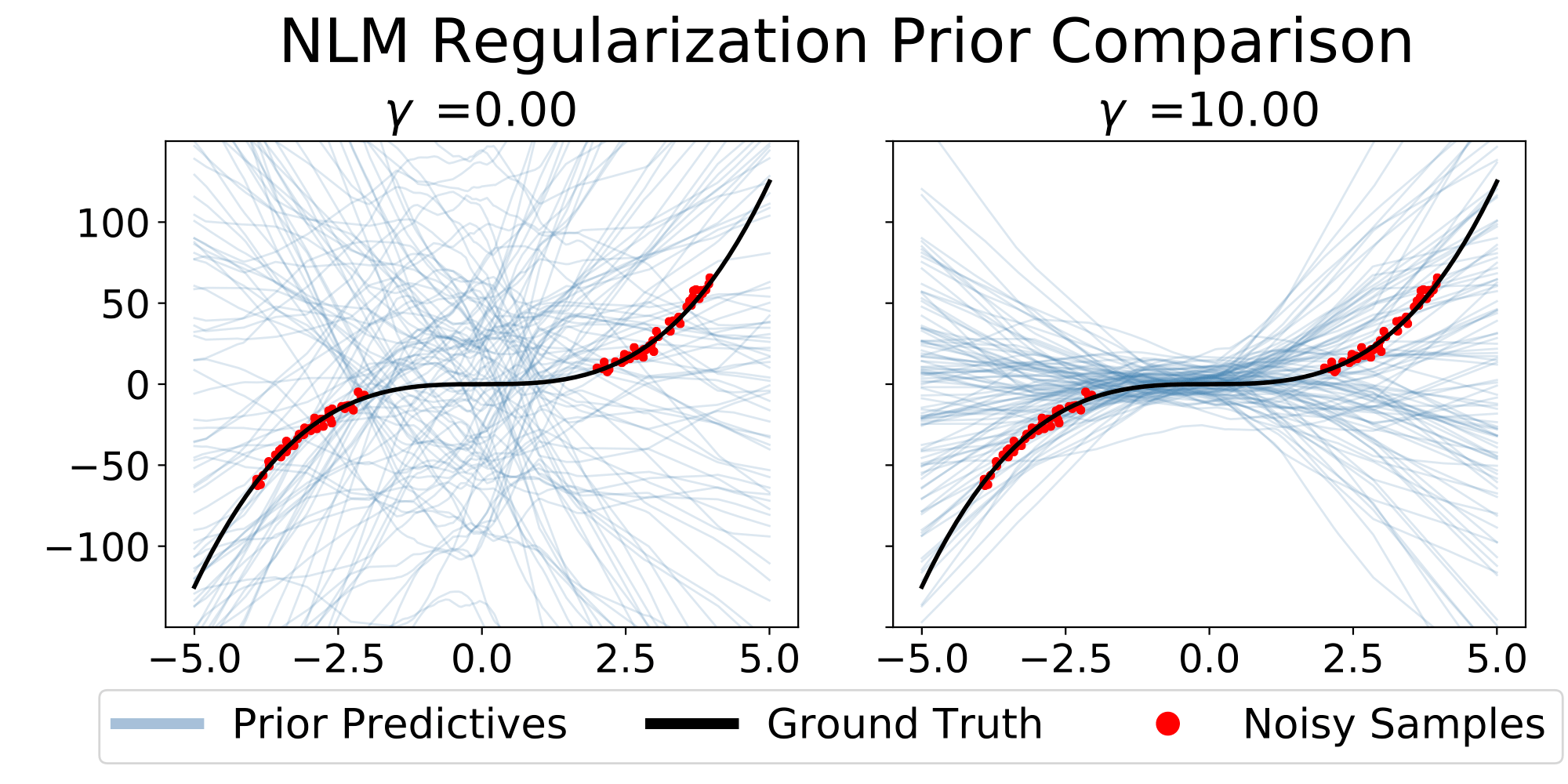}
    \caption{Samples from the prior predictive.}
    \label{fig:nlm_priors}
    \end{subfigure}
    ~
    \begin{subfigure}{1.0\linewidth}
    \vspace{0.5cm}
    \centering
    \includegraphics[width=0.85\linewidth]{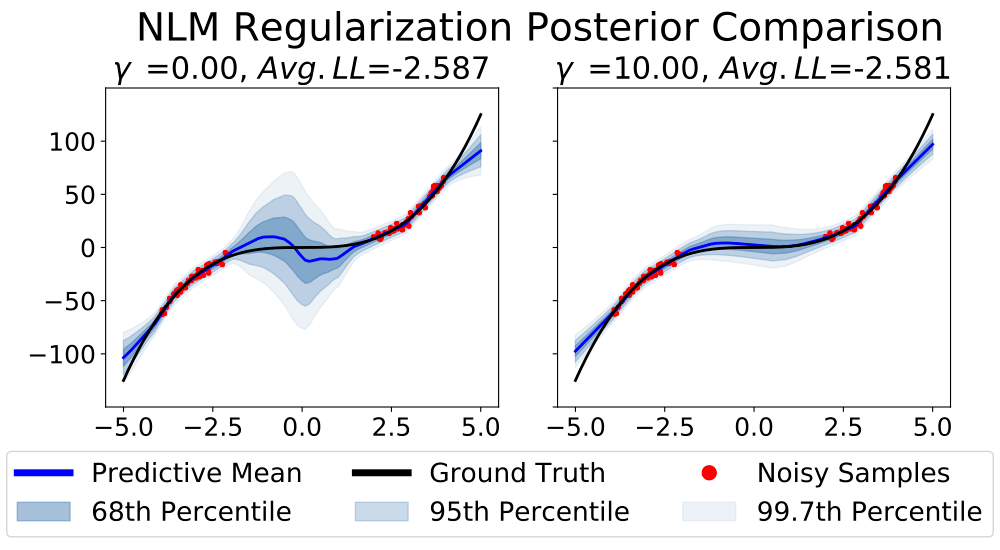}
    \caption{Posterior predictive.}
    \label{fig:nlm_posteriors}
    \end{subfigure}
    
    \caption{\textbf{MAP training discourages functional diversity under the NLM prior predictive, preventing it from capturing in-between uncertainty.} Top-row: samples from the prior predictive. Bottom-row: posterior predictive. Left-column: regularization of $\gamma = 0.0$. Right-column: $\gamma = 10.0$. More regularization therefore discourages functional diversity and results in an NLM unable to capture in-between uncertainty. We note that trivially setting $\gamma = 0.0$ (i.e. using MLE training) will not fix the issue. Figure \ref{fig:rr_reg} shows that MLE training does not consistently result in functional diversity under the prior predictive. Lastly, this figure illustrates a general pathology of log-likelihood based training objectives, since log-likelihood cannot be used to evaluate uncertainty~\cite{uncertainty_quality}; the log-likelihood for both posterior predictives is nearly the same, even though only one of the two expresses in-between uncertainty.}
    \label{fig:nlm_map}
\end{figure}

\paragraph{Failure of MLE Training.}
While MAP training explicitly discourages functional diversity under the prior predictive, MLE training does not; however, it also does nothing to encourage it, thereby leaving diversity up to chance.
As a result, across random restarts, MLE training ($\gamma = 0$) rarely learns models able to distinguish between data-poor and data-rich regions (see Figure \ref{fig:rr_reg} of Appendix \ref{sec:appendix_toy_results}).
As we discuss below, the failure of MLE training to encourage functional diversity actually affects all three training objectives -- MLE, MAP, and marginal likelihood training -- as well as hyper-parameter selection.  

\paragraph{Failure of Marginal Likelihood Training.}
When $\gamma > 0$, just like MAP training, Marginal Likelihood training (Equation \ref{eq:marginal_ll}) discourages learning models with expressive prior predictive.
Specifically, the feature bases learned by optimizing $\mathcal{L}_{\mathrm{Marginal}}$ do not span diverse functions across random restarts; hence the corresponding posterior predictive distributions are inexpressive and are unable to capture in-between uncertainty (Appendix \ref{sec:appendix_toy_results} Figure \ref{fig:marg_highgamma}).
When $\gamma$ is very close to $0$, just like MLE training, the learned feature bases are rarely expressive across random restart (see Appendix \ref{sec:appendix_toy_results} Figure \ref{fig:marg_lowgamma}).

Most interestingly, however, when $\gamma$ is exactly $0$, we show that Marginal Likelihood training suffers from a new failure: the feature map blows up for ReLU networks. 
Intuitively, increasing the magnitude of $\mathbf{\Phi_\theta}$ by a scalar multiple allows us to decrease the weights of the last layer by the same multiple with no loss to the likelihood, and thus we can trivially increase $\mathcal{L}_{\mathrm{Marginal}}$ by scaling the feature basis $\mathbf{\Phi_\theta}$. We formalize this intuition in the proposition below.
\begin{restatable}{proposition}{ThmMarginalBlowup} \label{thm:marg_blowup}~ 
Suppose ReLU activations and that $\Phi_\theta^\intercal \Phi_\theta$ is invertible. For fixed $\theta$, $\mathbf{w}$ and any $c >0$, we define $\theta^c$ as $\theta$ but with the last layer of weights scaled by $c$, we also define $\mathbf{w}^c = \frac{1}{c} \mathbf{w}$. For a sufficiently large $C>0$ and any $c> C$, we have that $\mathcal{L}_{\mathrm{Marginal}}(\theta_\mathrm{Full}) < \mathcal{L}_{\mathrm{Marginal}}(\theta^c_\mathrm{Full})$, where $\theta_\mathrm{Full} = (\theta, \mathbf{w})$ and $\theta^c_\mathrm{Full} = \left(\theta^c, \mathbf{w}^c\right)$.
\end{restatable}
The proof can be found in Appendix \ref{sec:marginal_ll}.
Proposition \ref{thm:marg_blowup} tells us that that we can continue to increase $\mathcal{L}_{\mathrm{Marginal}}$ by scaling any solution $(\theta, \mathbf{w})$ by larger and larger $c$'s -- that is, marginal log-likelihood training of $\theta$ is incentivized to trivially increase the magnitude of the feature basis $\mathbf{\Phi_\theta}$ rather than meaningfully change $\theta$. This blow-up of the learned features $\mathbf{\Phi_\theta}$ necessitates adding a regularization term $\gamma\norm{\theta}_2^2$ to the marginal log-likelihood objective, which then limits the expressiveness of the feature bases (and hence the expressiveness of the posterior predictive distribution).

\paragraph{General Failure of Log-Likelihood Based Training.}
As observed by ~\cite{uncertainty_quality}, log-likelihood does not measure the quality of a model's uncertainty in data-poor regions.
As such, learning a basis by maximizing the likelihood of the observed data (whether via the MLE, MAP, or marginal likelihood objectives) does not encourage learning a basis spanning a diverse class of functions, and thus does not ensure an expressive prior (and hence posterior) predictive. 
For example, in Figure \ref{fig:nlm_posteriors}, the test log-likelihood of both models are nearly equivalent, yet one is uncertain in the gap where the other is nearly as confident as on the observed data.

The fact that log-likelihood cannot be used to evaluate the uncertainty of the model in data-poor region additionally presents problems for \emph{all three} training objectives when selecting hyper-parameters: e.g. $\gamma$ (the regularization strength) and $L$ (the number of features in the feature map $\phi_\theta$), or the architecture of $\phi_\theta$ (e.g. the depth and width of the neural network).
For example, $\gamma$ is typically chosen via grid-search or BayesOpt to maximize the log-likelihood on the validation set (sampled from the same distribution as the training data). However, in Figure \ref{fig:nlm_posteriors}, the test log-likelihood of the NLM with the regularized feature map happens to be a hair higher. Thus, by maximizing validation log-likelihood, we may choose a model with a prior predictive that is inexpressive over the data-scarce region and hence unable to capture in-between uncertainty in the posterior predictive.

Similarly, we find that if we use validation log-likelihood to select the architecture of $\phi_\theta$, we are likely to choose models that are unnecessarily large.
In Figures \ref{fig:rr_features} and \ref{fig:rr_depth} of Appendix \ref{sec:appendix_toy_results}, we examine the effect of the depth of the network, as well as the number of features $L$, on the expressiveness of the prior predictives (and hence the posterior predictives) of NLMs. In particular, we show that for shallow and narrow models, even when $\phi_\theta$ is unregularized, random restarts consistently result in models with inexpressive prior predictives. On the other hand, for models with more capacity, some random restarts do yield expressive priors predictives. Thus, traditional NLM training and log-likelihood based hyper-parameter selection hinder us from (1) training modest-sized models (i.e. with a small $L$) that can express in-between uncertainty, and (2) from selecting amongst models with a larger capacity that, by chance, express in-between uncertainty. 

This general failure of log-likelihood based objectives motivates  our proposed training framework (discussed next), which avoids the failure modes of the three traditional NLM training objectives by explicitly encouraging for functional diversity under the prior predictive. This allows us to learn modest-sized models with expressive posterior predictive distributions that can distinguish between data-poor and data-rich regions.

\section{Training Framework: Uncertainty-Aware Bases via Auxiliary Networks} \label{sec:method}

\begin{figure}
  \centering
  \includegraphics[width=0.85\linewidth]{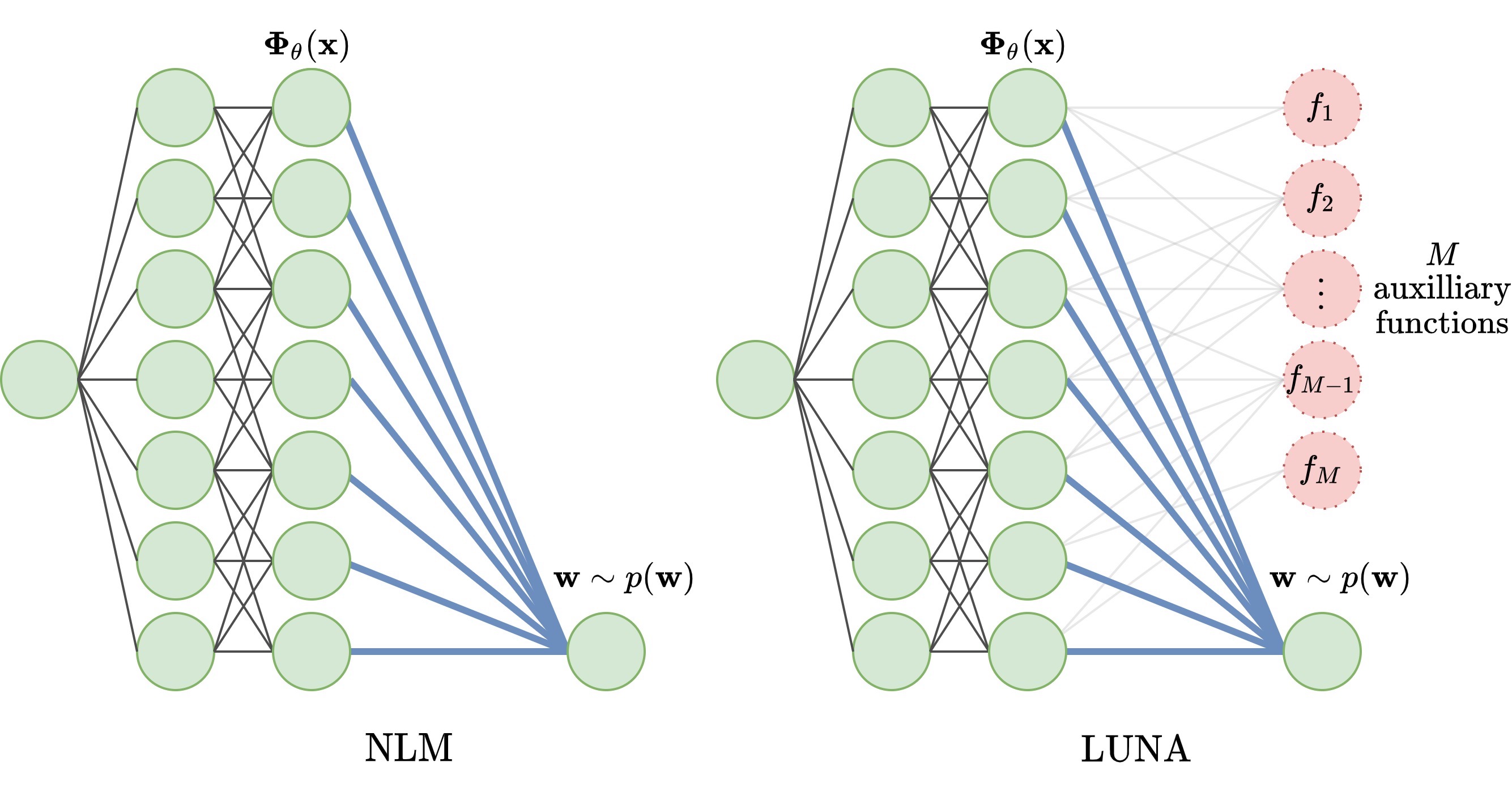}
  \caption{\textbf{Illustration of traditional NLM training (left) and UNA training (right).} Black arrows are deterministic weights, blue arrows are probabilistic weights (i.e. weights for which we assume a prior), and red nodes are auxiliary regressors. In Step I of UNA, the auxiliary regressors are trained to express a diverse set of functions. In Step II, they are discarded and Bayesian linear regression is performed on the learned basis $\mathbf{\Phi}_\theta$.}
  \label{fig:lb_nlm}
\end{figure}

We propose a general training framework, UNcertainty Aware Training (UNA), to learn NLMs that satisfy task-specific desiderata (e.g. to learn an NLM that captures in-between uncertainty). 
Since it is unintuitive to directly train the feature-basis of an NLM to satisfy task-specific desiderata (i.e. to specify desiderata in ``feature-space''), we instead choose to specify our desiderata in ``function-space'' (i.e. over functions that this basis spans). 
We do this by training a set of auxiliary regressors on a shared feature basis to satisfy the properties needed for good downstream performance. 
Like traditional NLM training, UNA training consists of two steps:


\textbf{Step I: Feature Training with Diverse and Task-Appropriate Auxiliary Regressors.} 
We train $M > 1$ auxiliary linear regressors, $f_m(\mathbf{x}) = \mathbf{\Phi_\theta}\mathbf{\widetilde{w}}_m$, on a shared feature basis $\mathbf{\Phi_\theta}$ (see illustration in Figure \ref{fig:lb_nlm}).
In this way, we can indirectly specify our desiderata on the feature-basis via an appropriately chosen objective function applied to the regressors.
By training the auxiliary regressors $f_m(x)$ to be diverse, for example, the shared feature basis will support a prior predictive distribution over a diverse class of functions. Furthermore, one can impose constraints on the auxiliary regressors expressing domain or functional knowledge. 

\textbf{Step II: Bayesian Linear Regression on Features.} After optimizing the feature map, \emph{we discard the $M$ auxiliary regressors} $\mathbf{\widetilde{w}}_m$ and perform Bayesian linear regression on the expressive feature basis $\mathbf{\Phi_\theta}$ learned in Step I. That is, we infer the posterior $p(\mathbf{w} | \mathcal{D}, \theta)$.
While Step I of UNA is novel, Step II directly follows from traditional NLM training.

Since this framework is quite general, we now propose a concrete instantiation, LUNA, designed to train NLMs that capture in-between uncertainties.
Our instantiation is both scalable and easy to implement.
In Section \ref{sec:experiments}, we then show that our instantiation produces expressive posterior predictive distributions that capture in-between uncertainties.

\subsection{Learned Uncertainty-Aware (LUNA) Bases} \label{sec:luna}
Using the insights from Section \ref{sec:nlm_pathologies}, we propose an instantiation of UNA.
In this instantiation, we explicitly encourage the feature basis to span diverse functions under the prior $p(\mathbf{w})$ (i.e. we do not leave diversity up to chance, as with MLE training). 
To do this, we design a training objective $\mathcal{L}_\text{LUNA}$ for Step I that maximizes the average log-likelihood of the auxiliary regressors on the training data, measured by $\mathcal{L}_\text{FIT}$, while encouraging for functional diversity amongst them, measured by $\mathcal{L}_\text{DIVERSE}$ (defined further below):
\begin{equation}
\mathcal{L}_\text{LUNA}(\Psi) = \mathcal{L}_\text{FIT}(\Psi) - \lambda \cdot \mathcal{L}_\text{DIVERSE}(\Psi),
\end{equation}
where $\Psi = (\theta, \mathbf{\tilde{w}}_1, \hdots \mathbf{\tilde{w}}_M)$, $\theta$ parameterizes the shared design matrix and $\mathbf{\tilde{w}}_m$ are the weights of the auxiliary regressor $f_m =  \mathbf{\Phi_\theta}\mathbf{\tilde{w}}_m$. The constant $\lambda$ controls for the degree to which we prioritize diversity. 
After optimizing our feature map via: 
$$
\theta_\text{LUNA}, \{\mathbf{\tilde{w}}_{m}^*\}  = \mathrm{argmax}_\Psi\; \mathcal{L}_\text{LUNA}(\Psi),
$$
we discard the auxiliary regressors $\{\mathbf{\tilde{w}}_{m}^*\}$ and perform Bayesian linear regression on the diversified feature basis, the LUNA basis. That is, we analytically infer the posterior $p(\mathbf{w} | \mathcal{D}, \theta_{\text{LUNA}})$ over the Bayesian last layer of weights $\mathbf{w}$ in the NLM (as in Step II of traditional NLM training).  
\emph{In summary, LUNA training results in a basis that supports a diverse set of predictions by varying $\mathbf{w}$.}

\paragraph{$\mathcal{L}_\text{FIT}$: Fitting the Auxiliary Regressors.}
We learn the regressors jointly with $\mathbf{\Phi_\theta}$, by maximizing the average training log-likelihood of the regressors on the training data, with $\ell_2$ penalty on $\theta$ as well as on the weights of each regressor:
\begin{equation*}
\mathcal{L}_\text{FIT}(\Psi) = \frac{1}{M}\sum_{m=1}^M \log \mathcal{N}(\mathbf{y}; f_m(\mathbf{x}), \sigma^2 \mathbf{I})- \gamma\norm{\Psi}_2^2.
\end{equation*}

\paragraph{$\mathcal{L}_\text{DIVERSE}$: Enforcing  diversity.}
We enforce diversity in the auxiliary regressors as a proxy for the diversity of the functions spanned by the feature basis. 
We adapt the Local Independence Training (LIT) objective proposed by \cite{diversity_enforcement} to encourage our regressors to extrapolate differently away from the training data,
\begin{equation*}
\mathcal{L}_\text{DIVERSE}(\Psi) = 
\sum_{i=1}^{M-1}\sum_{j=i+1}^M
\text{CosSim}^2\left(\nabla_{\mathbf{x}}f_i(\mathbf{x}), \nabla_{\mathbf{x}}f_j(\mathbf{x}) \right),
\label{eq:luna_obj}
\end{equation*}
where $\text{CosSim}$ is cosine similarity.
Intuitively, $\mathcal{L}_\text{DIVERSE}(\Psi)$ iterates over pairs of auxiliary regressors $f_i, f_j$, encouraging the angle between their gradients with respect to the inputs to be large. 
In doing so, this penalty encourages that different regressors extrapolate differently away from the data. 
We avoid expensive gradient computations using a finite difference approximation -- see Appendix \ref{luna-training-objective} for a detailed explanation of $\mathcal{L}_\text{DIVERSE}(\Psi)$.

\paragraph{Model Selection.} Since in Section \ref{sec:nlm_pathologies}, we show that log-likelihood cannot be used to distinguish between models that do and do not capture in-between uncertainty, we cannot use log-likelihood (alone) for hyper-parameter selection. 
As such, we incorporate the diversity penalty into the model selection process as follows.
After training, we scale the diversity penalty by $1/{M \choose 2}$, which is the number of combinations of auxiliary regressors, so that the diversity penalty is comparable across choices of $M$.
Then, across all hyper-parameter choices and random restarts, we keep only models that score in the top $10\%$ on validation log-likelihood, out of which we select the hyper-parameters of models that exhibit the most diversity (i.e. the lowest $\mathcal{L}_\text{DIVERSE}$).
In this way, we select models that both have good fit on the observed data, and capture in-between uncertainty. 

\paragraph{Demonstration on 1-D synthetic data.}
LUNA's auxiliary regressors and resultant posterior predictive are visualized in Figure \ref{fig:luna_aux}, on the ``Cubic Gap Example'' (described in Appendix \ref{sec:synthetic-data}). The figure shows that the regressors both fit the data and extrapolate differently, and therefore the resultant model expresses in-between uncertainty. 
In contrast to MLE training, LUNA \emph{consistently} captures in-between uncertainty across a varying number of auxiliary regressors $L$ and across random restarts (see Appendix \ref{sec:appendix_toy_results} Figure \ref{fig:aux_rr}). 
That is, unlike traditional NLM training, which struggles to capture in-between uncertainty (especially for smaller architectures), LUNA training better utilizes the available capacity of the NLM to fit the data and express uncertainty. 

\begin{figure}[t]
    
    \centering
    \includegraphics[width=0.8\linewidth]{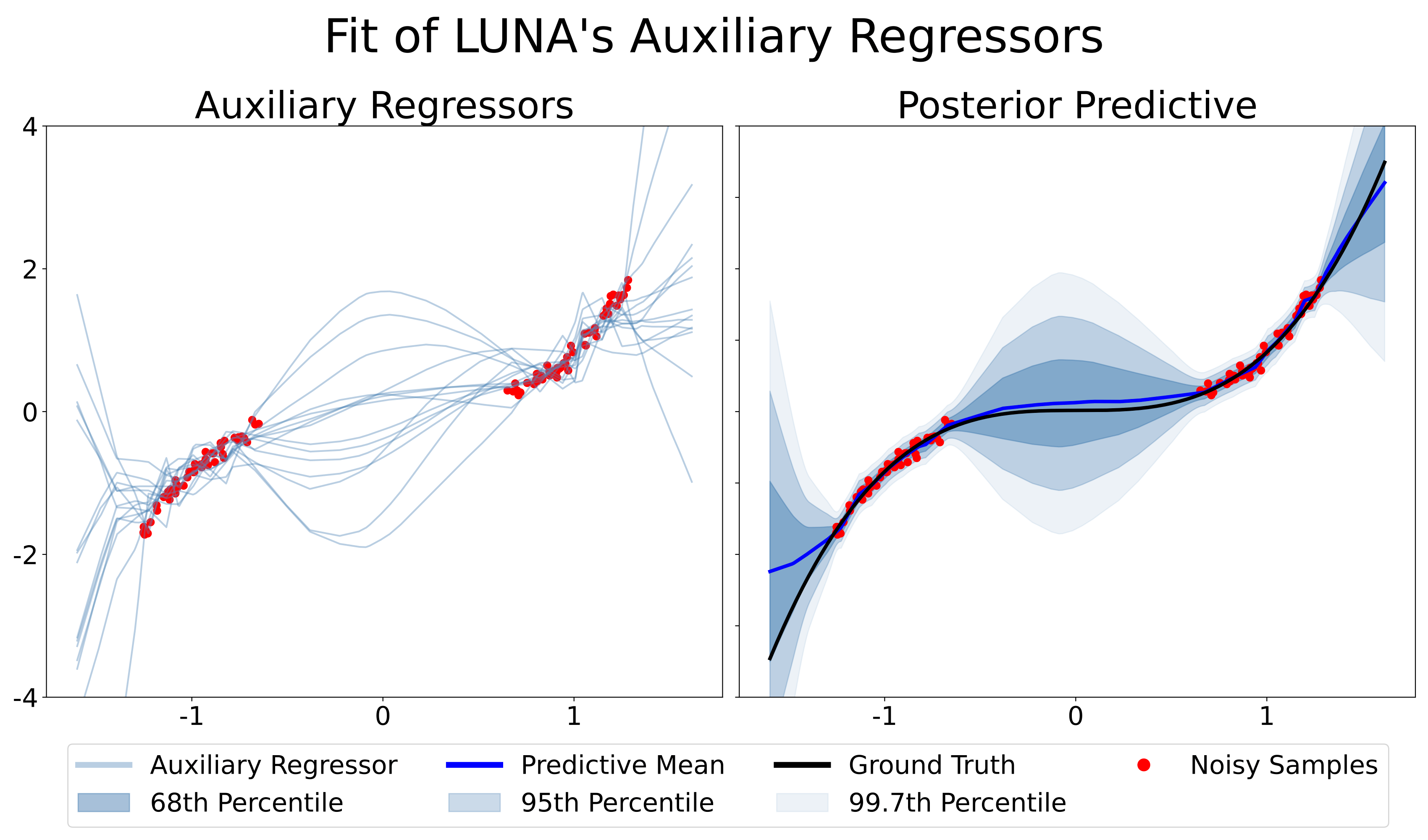}
    
    \caption{\textbf{LUNA training yields bases that explain the data and extrapolate differently, thereby capturing in-between uncertainty.}}
    \label{fig:luna_aux}
\end{figure}

\section{Experiments}\label{sec:experiments}

We compare LUNA to traditional NLM training as well as to a variety of baselines on 3 synthetic and on 6 UCI ``gap'' datasets \citep{uci_gap}. We show that LUNA bases can distinguish data-scarce regions from data-rich regions where baseline methods struggle. 
We then show that as a result, LUNA offers competitive performance relative to baselines on a variety of downstream tasks. On the transfer learning task, we show that with fewer features, LUNA bases achieve lower generalization error and retain their utility under covariate shift. On Bayesian Optimization tasks, we show that LUNA bases are comparable to baselines. 

\paragraph{Baselines.} In our experiments, we consider several baselines. 
First, we compare against NLMs with traditional inference (NLM).
Second, when it is not too computationally taxing, we compare LUNA bases against Gaussian Processes (GP) and BNNs with HMC inference (HMC). We regard these as our ``gold standard'', since, with appropriately chosen hyper-parameters, both explain the data well and capture in-between uncertainty. When using small architectures, we also compare against BNNs with Linearised Laplace (Lin Lap) inference \citep{uci_gap} (which scales poorly to large architectures). 
Third, we compare against a recently proposed method that, like an NLM, consists of a Bayesian regressor trained on-top of a deterministic feature extractor: Spectral-normalized Neural Gaussian Processes using a GP final layer (SNGP) as well as using a Random Fourier Features regression final layer (RFF SNGP) \citep{liu2020simple}.
Lastly, we compare against other existing models and inference for capturing uncertainty: MC Dropout (MCD) \citep{mcd}, BNN with mean-field variational inference (BBVI), Vanilla Ensembles (ENS VAN), Anchored Ensembles (ENS ANC) \citep{pearce2020uncertainty}, and Bootstrap Ensembles (ENS BOOT). 
Experimental setup detailed in Appendix \ref{sec:exp-setup}.

\paragraph{Evaluation Metrics.} In addition to using downstream tasks to compare the quality of LUNA's predictions and predictive uncertainty against the aforementioned baselines, we also use the following metrics: Average Log-Likelihood (LL) and Root Mean Square Error (RMSE) to assess model fit, and gap to not-gap Epistemic Uncertainty Relative Change (EURC), to assess whether a model's predictive uncertainty is higher where there is no data (relative to where there is data). All metrics are described in Appendix \ref{sec:eval-metrics}. 

\subsection{LUNA bases captures in-between uncertainty on synthetic and real data}

%
%
\begin{figure}
    \centering
    \begin{subfigure}{1.0\textwidth}
        \includegraphics[width=1.0\linewidth]{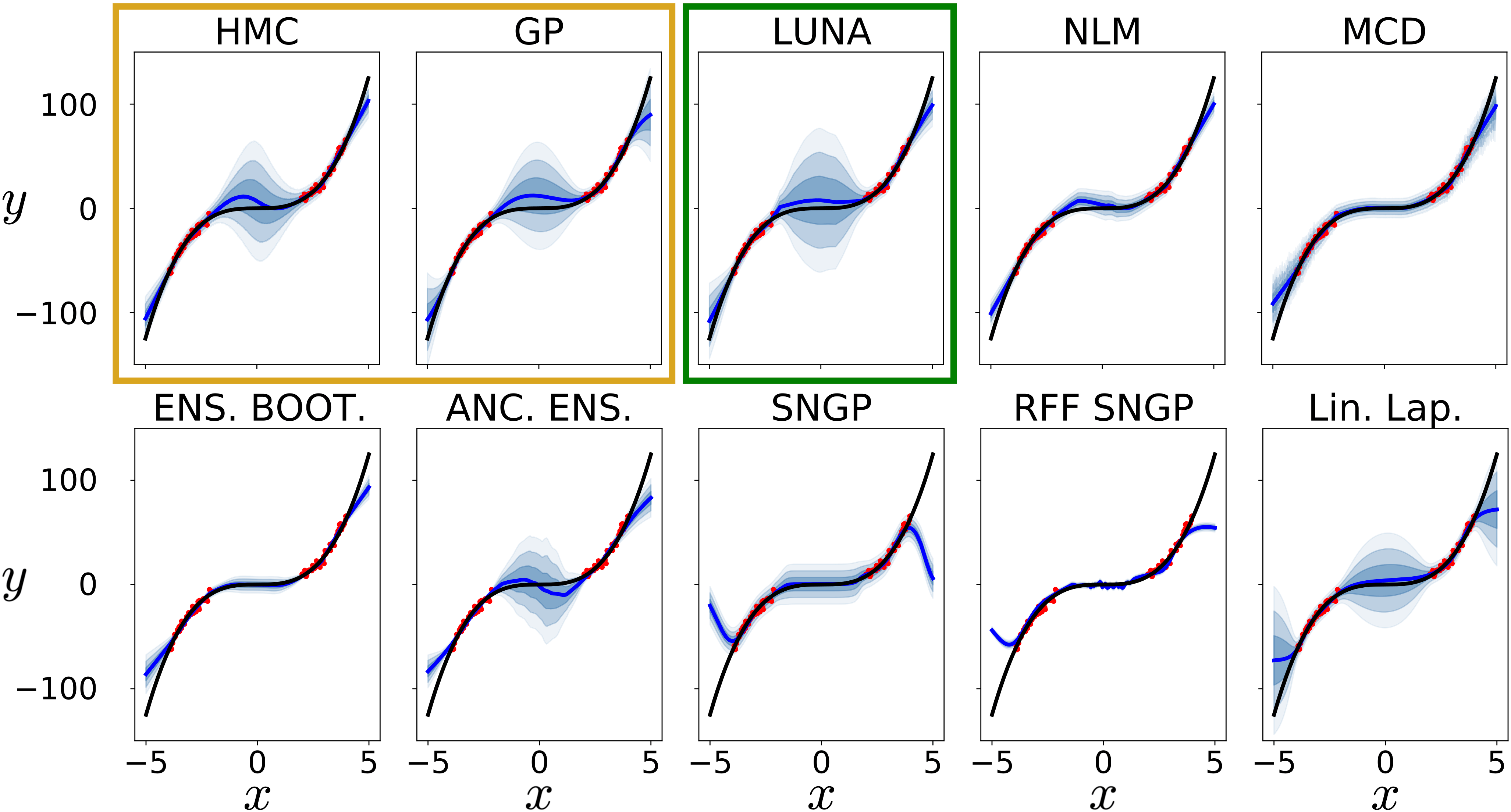}
    \end{subfigure}
    
    \vspace{0.25cm}
    
    \begin{subfigure}{1.0\textwidth}
        \raggedleft 
        \includegraphics[width=0.913\linewidth]{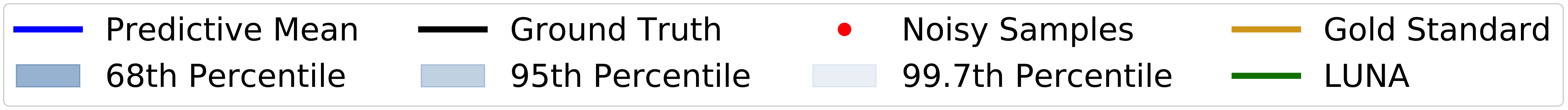}
    \end{subfigure}
    
    \caption{\textbf{LUNA's in-between uncertainty resembles that of ``gold standard'' while most other baselines' do not.} On the ``Cubic Gap'' data (described in Appendix \ref{sec:synthetic-data}), LUNA, along with Anchored Ensembles and Linearised Laplace, most resembles the gold standard baselines -- GP, and BNN with HMC -- while the remaining baselines exhibit little uncertainty in the gap.}
    \label{fig:toy_dataset_results}
\end{figure}

\paragraph{LUNA training captures in-between uncertainty on toy data.} We construct a synthetic 1-D ``Cubic Gap'' regression data set, sampled from a cubic function with a gap in the training input (detailed in Appendix \ref{sec:synthetic-data}).
Figure \ref{fig:toy_dataset_results} shows that LUNA bases, along with Anchored Ensembles and Linearised Laplace, matches the performance of the ``gold-standards'' -- BNN with HMC and GP -- while other baselines either drastically underestimate in-between uncertainty or express uncertainties that do not scale with respect to the distance to distance to the training data (e.g. RFF SNGP). While Linearised Laplace performs well on small toy datasets, unlike other baselines, it does not scale to larger networks, and has difficulties with various activation functions~\citep{uci_gap}. As such, we do not include it in the remaining experiments.

\begin{table*}
    \centering
    \text{Gap to Not-Gap Avg. Epistemic Uncertainty Relative Change}
    \resizebox{\textwidth}{!}{
    \begin{tabular}{ccccccc}
        \toprule
        & Yacht - FROUDE & Concrete - CEMENT
        & Concrete - SUPER & Boston - RM
        & Boston - LSTAT  & Boston - PTRATIO\\
		\cmidrule(lr){2-2} \cmidrule(lr){3-3} \cmidrule(lr){4-4} \cmidrule(lr){5-5} \cmidrule(lr){6-6} \cmidrule(lr){7-7}
        ENS BOOT & 0.00\% $\pm$ 0.00\% & \textbf{40.85\% $\pm$ 7.60\%} & \textbf{87.64\% $\pm$ 15.79\%} & -9.22\% $\pm$ 9.57\% & 6.40\% $\pm$ 10.97\% & -1.33\% $\pm$ 8.13\% \\
        ENS VAN & \textbf{26.68\% $\pm$ 11.75\%} & \textbf{120.70\% $\pm$ 26.68\%} & \textbf{163.28\% $\pm$ 26.81\%} & -5.61\% $\pm$ 11.41\% & \textbf{18.84\% $\pm$ 11.47\%} & \textbf{39.83\% $\pm$ 16.37\%} \\
        ENS ANC & 8.66\% $\pm$ 23.98\% & \textbf{93.03\% $\pm$ 25.66\%} & \textbf{149.31\% $\pm$ 31.30\%} & -8.94\% $\pm$ 11.15\% & 10.03\% $\pm$ 10.72\% & \textbf{31.12\% $\pm$ 22.25\%} \\
        \cmidrule{1-7}
        NLM & \textbf{24.08\% $\pm$ 17.23\%} & -5.92\% $\pm$ 3.63\% & 4.50\% $\pm$ 5.89\% & -7.47\% $\pm$ 3.57\% & -5.79\% $\pm$ 7.53\% & -3.81\% $\pm$ 6.32\% \\
        GP & \textbf{71.73\% $\pm$ 19.75\%} & \textbf{75.59\% $\pm$ 15.12\%} & \textbf{119.48\% $\pm$ 19.55\%} & -16.34\% $\pm$ 6.76\% & -2.40\% $\pm$ 9.20\% & \textbf{17.74\% $\pm$ 13.17\%} \\
        MCD & -52.43\% $\pm$ 10.17\% & \textbf{6.86\% $\pm$ 5.92\%} & \textbf{11.95\% $\pm$ 7.30\%} & -16.09\% $\pm$ 9.55\% & -11.76\% $\pm$ 9.48\% & -5.22\% $\pm$ 6.88\% \\
        SNGP & 38.33\% $\pm$ 63.97\% & -6.05\% $\pm$ 4.39\% & -2.93\% $\pm$ 4.23\% & -10.61\% $\pm$ 3.88\% & -6.43\% $\pm$ 5.46\% & -6.92\% $\pm$ 5.94\% \\
        BBVI & -10.08\% $\pm$ 6.15\% & -18.15\% $\pm$ 5.58\% & 10.85\% $\pm$ 12.68\% & -13.74\% $\pm$ 5.15\% & -31.91\% $\pm$ 3.27\% & \textbf{20.30\% $\pm$ 8.54\%} \\
        \cmidrule{1-7}
        LUNA & \textbf{59.17\% $\pm$ 36.87\%} & \textbf{55.93\% $\pm$ 28.83\%} & \textbf{416.02\% $\pm$ 197.69\%} & -11.09\% $\pm$ 5.20\% & \textbf{29.09\% $\pm$ 18.15\%} & \textbf{60.60\% $\pm$ 45.32\%} \\
	\end{tabular}}
    \caption{\textbf{On UCI ``gap'' data, LUNA captures higher in-gap epistemic uncertainty where baselines struggle.} This table computes the gap to not-gap Epistemic Uncertainty Relative Change (EPRC), which is positive when epistemic uncertainty is higher in the gap than it is where there is data (and zero or negative otherwise). Bolded results indicate models that express higher uncertainty inside the gap compared to outside by at least one standard deviation. LUNA and Ensemble are the only ones that, consistently (on 5 of the 6 data-sets), have significantly more epistemic uncertainty in the gap (on ``Boston - RM'', all models struggle).}
    \label{tab:uci-gap-epistemic}
\end{table*}

\begin{table*}[t]
    \centering
    \text{Root Mean Square Error (Test)}
    \resizebox{\textwidth}{!}{
    \begin{tabular}{ccccccccccccc}
        \toprule
        & \multicolumn{2}{c}{Yacht - FROUDE} & \multicolumn{2}{c}{Concrete - CEMENT} 
        & \multicolumn{2}{c}{Concrete - SUPER} & \multicolumn{2}{c}{Boston - RM}
        & \multicolumn{2}{c}{Boston - LSTAT} & \multicolumn{2}{c}{Boston - PTRATIO} \\
		\cmidrule(lr){2-3} \cmidrule(lr){4-5} \cmidrule(lr){6-7} \cmidrule(lr){8-9} \cmidrule(lr){10-11} \cmidrule(lr){12-13}
		& Not Gap & Gap & Not Gap & Gap & Not Gap & Gap & Not Gap & Gap & Not Gap & Gap & Not Gap & Gap \\
		\cmidrule{2-13}
		ENS BOOT & 1.21 $\pm$ 0.47 & 0.51 $\pm$ 0.08 & 5.27 $\pm$ 0.97 & 6.04 $\pm$ 0.10 & 4.70 $\pm$ 0.93 & 7.55 $\pm$ 0.26 & 2.81 $\pm$ 0.86 & 3.07 $\pm$ 0.10 & 3.28 $\pm$ 1.04 & 3.47 $\pm$ 0.14 & 3.46 $\pm$ 0.87 & 3.20 $\pm$ 0.08 \\
		ENS VAE & 0.84 $\pm$ 0.39 & 0.40 $\pm$ 0.04 & 5.06 $\pm$ 0.90 & 6.10 $\pm$ 0.18 & 4.44 $\pm$ 0.78 & 7.49 $\pm$ 0.18 & 2.78 $\pm$ 0.90 & 3.04 $\pm$ 0.08 & 3.12 $\pm$ 1.12 & 3.20 $\pm$ 0.13 & 3.41 $\pm$ 0.77 & 3.27 $\pm$ 0.08 \\
		ENS ANC & 0.91 $\pm$ 0.33 & 0.86 $\pm$ 0.07 & 5.54 $\pm$ 0.87 & 6.36 $\pm$ 0.22 & 5.33 $\pm$ 0.69 & 7.77 $\pm$ 0.46 & 2.90 $\pm$ 0.74 & 3.09 $\pm$ 0.12 & 3.30 $\pm$ 1.17 & 3.33 $\pm$ 0.15 & 3.40 $\pm$ 0.87 & 3.16 $\pm$ 0.07 \\
		\cmidrule{1-13}
		NLM & 0.65 $\pm$ 0.26 & 0.72 $\pm$ 0.14 & 5.31 $\pm$ 0.97 & 7.01 $\pm$ 0.49 & 4.60 $\pm$ 0.94 & 8.47 $\pm$ 0.36 & 3.02 $\pm$ 0.87 & 3.21 $\pm$ 0.11 & 3.69 $\pm$ 1.51 & 3.93 $\pm$ 0.35 & 3.70 $\pm$ 0.67 & 3.68 $\pm$ 0.12 \\
		GP & 1.89 $\pm$ 0.54 & 1.37 $\pm$ 0.21 & 6.01 $\pm$ 0.89 & 6.21 $\pm$ 0.07 & 5.91 $\pm$ 0.78 & 8.01 $\pm$ 0.16 & 3.10 $\pm$ 0.91 & 3.27 $\pm$ 0.19 & 3.52 $\pm$ 1.16 & 3.32 $\pm$ 0.15 & 3.45 $\pm$ 0.78 & 3.28 $\pm$ 0.04 \\
		MCD & 0.89 $\pm$ 0.31 & 6.78 $\pm$ 0.37 & 5.09 $\pm$ 1.07 & 7.27 $\pm$ 0.40 & 4.80 $\pm$ 0.91 & 7.93 $\pm$ 0.34 & 3.45 $\pm$ 1.15 & 3.17 $\pm$ 0.10 & 3.40 $\pm$ 1.09 & 4.08 $\pm$ 0.36 & 3.69 $\pm$ 1.02 & 3.27 $\pm$ 0.18 \\
		SNGP & 1.04 $\pm$ 0.68 & 1.31 $\pm$ 1.11 & 5.15 $\pm$ 0.74 & 5.93 $\pm$ 0.17 & 5.00 $\pm$ 0.69 & 7.33 $\pm$ 0.32 & 3.07 $\pm$ 0.56 & 3.41 $\pm$ 0.18 & 3.79 $\pm$ 1.02 & 4.18 $\pm$ 0.26 & 3.75 $\pm$ 1.00 & 3.77 $\pm$ 0.21 \\
		BBVI & 17.27 $\pm$ 5.87 & 30.05 $\pm$ 2.99 & 5.68 $\pm$ 0.80 & 6.36 $\pm$ 0.07 & 24.17 $\pm$ 7.56 & 54.56 $\pm$ 4.58 & 3.47 $\pm$ 0.87 & 3.53 $\pm$ 0.05 & 3.76 $\pm$ 1.22 & 3.82 $\pm$ 0.12 & 9.16 $\pm$ 3.29 & 30.52 $\pm$ 2.84 \\
		\cmidrule{1-13}
        LUNA & 1.16 $\pm$ 0.42 & 0.57 $\pm$ 0.10 & 5.50 $\pm$ 1.32 & 7.12 $\pm$ 0.36 & 4.92 $\pm$ 0.66 & 10.13 $\pm$ 1.05 & 3.34 $\pm$ 1.09 & 3.17 $\pm$ 0.38 & 3.57 $\pm$ 1.44 & 3.92 $\pm$ 0.26 & 3.58 $\pm$ 1.09 & 3.34 $\pm$ 0.14 \\
	\end{tabular}}
	\caption{\textbf{On UCI ``gap'' data, LUNA has comparable RMSE to baselines.}}
	\label{tab:uci-gap-rmse}
\end{table*}

\begin{table*}[t]
    \centering
    \text{Avg. Log-Likelihood (Test)}
    \resizebox{\textwidth}{!}{
    \begin{tabular}{ccccccccccccc}
        \toprule
        & \multicolumn{2}{c}{Yacht - FROUDE} & \multicolumn{2}{c}{Concrete - CEMENT} 
        & \multicolumn{2}{c}{Concrete - SUPER} & \multicolumn{2}{c}{Boston - RM}
        & \multicolumn{2}{c}{Boston - LSTAT} & \multicolumn{2}{c}{Boston - PTRATIO} \\
		\cmidrule(lr){2-3} \cmidrule(lr){4-5} \cmidrule(lr){6-7} \cmidrule(lr){8-9} \cmidrule(lr){10-11} \cmidrule(lr){12-13}
		& Not Gap & Gap & Not Gap & Gap & Not Gap & Gap & Not Gap & Gap & Not Gap & Gap & Not Gap & Gap \\
		\cmidrule{2-13}
		ENS BOOT & N/A & N/A & N/A & N/A & N/A & N/A & N/A & N/A & N/A & N/A & N/A & N/A \\
		ENS VAN & N/A & N/A & N/A & N/A & N/A & N/A & N/A & N/A & N/A & N/A & N/A & N/A \\
		ENS ANC & N/A & N/A & N/A & N/A & N/A & N/A & N/A & N/A & N/A & N/A & N/A & N/A \\
		\cmidrule{1-13}
		NLM & -1.29 $\pm$ 0.90 & -1.45 $\pm$ 0.52 & -3.15 $\pm$ 0.26 & -3.71 $\pm$ 0.19 & -2.97 $\pm$ 0.25 & -4.29 $\pm$ 0.16 & -2.56 $\pm$ 0.30 & -2.58 $\pm$ 0.03 & -2.82 $\pm$ 0.70 & -2.79 $\pm$ 0.12 & -2.73 $\pm$ 0.25 & -2.72 $\pm$ 0.04 \\
		GP & -1.76 $\pm$ 0.30 & -1.56 $\pm$ 0.04 & -3.18 $\pm$ 0.14 & -3.17 $\pm$ 0.01 & -3.19 $\pm$ 0.19 & -3.40 $\pm$ 0.01 & -2.53 $\pm$ 0.12 & -2.60 $\pm$ 0.02 & -2.62 $\pm$ 0.18 & -2.61 $\pm$ 0.02 & -2.63 $\pm$ 0.19 & -2.63 $\pm$ 0.01 \\
		MCD & -1.13 $\pm$ 0.27 & -32.49 $\pm$ 12.07 & -2.94 $\pm$ 0.18 & -3.54 $\pm$ 0.06 & -2.96 $\pm$ 0.18 & -3.71 $\pm$ 0.06 & -2.59 $\pm$ 0.22 & -2.52 $\pm$ 0.02 & -2.59 $\pm$ 0.20 & -2.69 $\pm$ 0.11 & -2.59 $\pm$ 0.16 & -2.58 $\pm$ 0.05 \\
		SNGP & -4.35 $\pm$ 5.66 & -8.14 $\pm$ 13.55 & -3.12 $\pm$ 0.21 & -3.35 $\pm$ 0.05 & -3.07 $\pm$ 0.20 & -3.87 $\pm$ 0.14 & -2.56 $\pm$ 0.17 & -2.65 $\pm$ 0.06 & -2.81 $\pm$ 0.39 & -2.90 $\pm$ 0.10 & -2.81 $\pm$ 0.39 & -2.77 $\pm$ 0.07 \\
		BBVI & -68.40 $\pm$ 31.42 & -207.53 $\pm$ 34.99 & -3.15 $\pm$ 0.17 & -3.30 $\pm$ 0.02 & -6.74 $\pm$ 2.24 & -27.14 $\pm$ 3.45 & -2.63 $\pm$ 0.09 & -2.64 $\pm$ 0.01 & -2.70 $\pm$ 0.21 & -2.63 $\pm$ 0.02 & -3.49 $\pm$ 0.35 & -8.14 $\pm$ 1.22 \\
		\cmidrule{1-13}
        LUNA & -2.82 $\pm$ 2.09 & -0.96 $\pm$ 0.16 & -3.14 $\pm$ 0.30 & -3.55 $\pm$ 0.13 & -3.00 $\pm$ 0.13 & -4.28 $\pm$ 0.36 & -2.56 $\pm$ 0.15 & -2.54 $\pm$ 0.03 & -2.72 $\pm$ 0.44 & -2.75 $\pm$ 0.07 & -2.69 $\pm$ 0.31 & -2.65 $\pm$ 0.06 \\
	\end{tabular}}
	\caption{\textbf{On UCI ``gap'' data, LUNA has comparable average log-likelihood to baselines.} That is, on the non-gap region, it has comparable log-likelihood to baselines, and on the gap region, its log-likelihood decreases since epistemic uncertainty is increased (see Table \ref{tab:uci-gap-epistemic}).}
	\label{tab:uci-gap-ll}
\end{table*}

\begin{table*}[t]
    \centering
    \text{Bayesian Optimization}
	\resizebox{\textwidth}{!}{
    \begin{tabular}{cc||c|cccccccc}
        \toprule
        Function & Steps & LUNA & GP & NLM & MCD & ENS ANC & ENS BOOT & ENS VAN & SNGP & BBVI \\
        \midrule
		branin & 50 & 
		0.01 $\pm$ 0.00 & 	 
		0.00 $\pm$ 0.00 & 	 
		0.01 $\pm$ 0.01 & 	 
		0.01 $\pm$ 0.01 & 	 
		0.01 $\pm$ 0.01 & 	 
		0.06 $\pm$ 0.12 & 	 
		0.01 $\pm$ 0.02 & 	 
		0.00 $\pm$ 0.00 & 	 
		0.01 $\pm$ 0.00  	 
		 \\ 
		hartmann6 & 200 & 
		0.32 $\pm$ 0.02 & 	 
		0.01 $\pm$ 0.00 & 	 
		0.57 $\pm$ 0.44 & 	 
		0.76 $\pm$ 0.25 & 	 
		0.23 $\pm$ 0.21 & 	 
		0.65 $\pm$ 0.28 & 	 
		0.68 $\pm$ 0.28 & 	 
		0.22 $\pm$ 0.00 & 	 
		0.71 $\pm$ 0.23 	 
		\\ 
		svm & 30 & 
		1.19 $\pm$ 0.12 & 	 
		1.20 $\pm$ 0.00 & 	 
		1.20 $\pm$ 0.06 & 	 
		1.10 $\pm$ 0.00 & 	 
		1.18 $\pm$ 0.12 & 	 
		1.11 $\pm$ 0.18 & 	 
		1.13 $\pm$ 0.05 & 	 
		1.19 $\pm$ 0.14 & 	 
		1.30 $\pm$ 0.35  	 
		 \\ 
		logistic & 30 & 
		7.64 $\pm$ 0.06 & 	 
		7.40 $\pm$ 0.00 & 	 
		7.64 $\pm$ 0.10 & 	 
		7.91 $\pm$ 0.29 & 	 
		7.66 $\pm$ 0.09 & 	 
		7.64 $\pm$ 0.08 & 	 
		7.59 $\pm$ 0.07 & 	 
		7.64 $\pm$ 0.07 & 	 
		7.92 $\pm$ 0.33  	 
		 \\ 
    \end{tabular}
    }
    \caption{\textbf{LUNA is comparable to baselines on Bayesian Optimization.} The table shows the different between the optima found via BayesOpt vs. the true optima $|f(x) - f(x^*)|$ across random restarts. Benchmarks compiled by \cite{eggensperger_toward}.}
    \label{tab:bayesopt}
\end{table*}

\paragraph{LUNA training capture in-between uncertainty in UCI gap datasets.} We use 6 UCI ``gap'' datasets \citep{uci_gap} with artificially created gaps in the training set (described in Appendix \ref{sec:uci-gap-examples}), and demonstrate that, as desired, LUNA's epistemic uncertainty distinguishes between the gap region and data-rich regions of the input space without significant decrease in predictive performance. 
Table \ref{tab:uci-gap-epistemic} shows that across all but one of the data-sets (on which all methods struggle), only LUNA and Vanilla Ensemble consistently learn models that are more uncertain on out-of-distribution (or in-gap) test data, relative to within distribution test data.
We evaluate both test log-likelihood (Table \ref{tab:uci-gap-ll}), and test RMSE (Table \ref{tab:uci-gap-rmse}) on data sampled inside vs. outside the gap to show that LUNA's fit is comparable to that of the remaining baselines.

\subsection{LUNA bases are competitive with baselines on downstream tasks}

\paragraph{LUNA bases are better for transfer learning, even with few features.} 
We assess how useful the learned feature bases are when the Bayesian regression model is retrained given new data -- specifically, data under covariate shift.
For this task, we compare LUNA against traditional NLM training, as well as against SNGP, since those are the only methods that consist of a Bayesian model trained on a deterministic feature map. 
We show that across all numbers of features in $\phi_\theta$, LUNA bases outperform both baselines. We construct a synthetic 1-D dataset (the ``Squiggle Gap'' data in Appendix \ref{sec:synthetic-data}), visualized in Figure \ref{fig:squiggle}, which unlike the ``Cubic Gap'' data, has unexpected variations in the held-out gap region. 
After training on the non-gap data, we fix the feature map $\phi_\theta$, and we infer the posterior $p(\mathbf{w} | \mathcal{D}, \theta)$ over the last layer using data from the gap. Figure \ref{fig:ll_vs_features_tf_relearn} shows that LUNA bases can easily be adapted to modeling data from the gap, while feature bases of baseline methods struggle to adapt, even as we increase the number of features.

\paragraph{LUNA training is comparable to baselines on BayesOpt.} We compare LUNA training against baselines on 4 Bayesian optimization benchmark tasks (see BayesOpt in Appendix \ref{sec:synthetic-data}).
In Table \ref{tab:bayesopt}, we see that LUNA's performance is comparable to baselines.

\begin{figure}[!t]
    \centering
    
  \begin{subfigure}{0.48\textwidth}
  \centering
    \centering
    \includegraphics[width=1.0\linewidth]{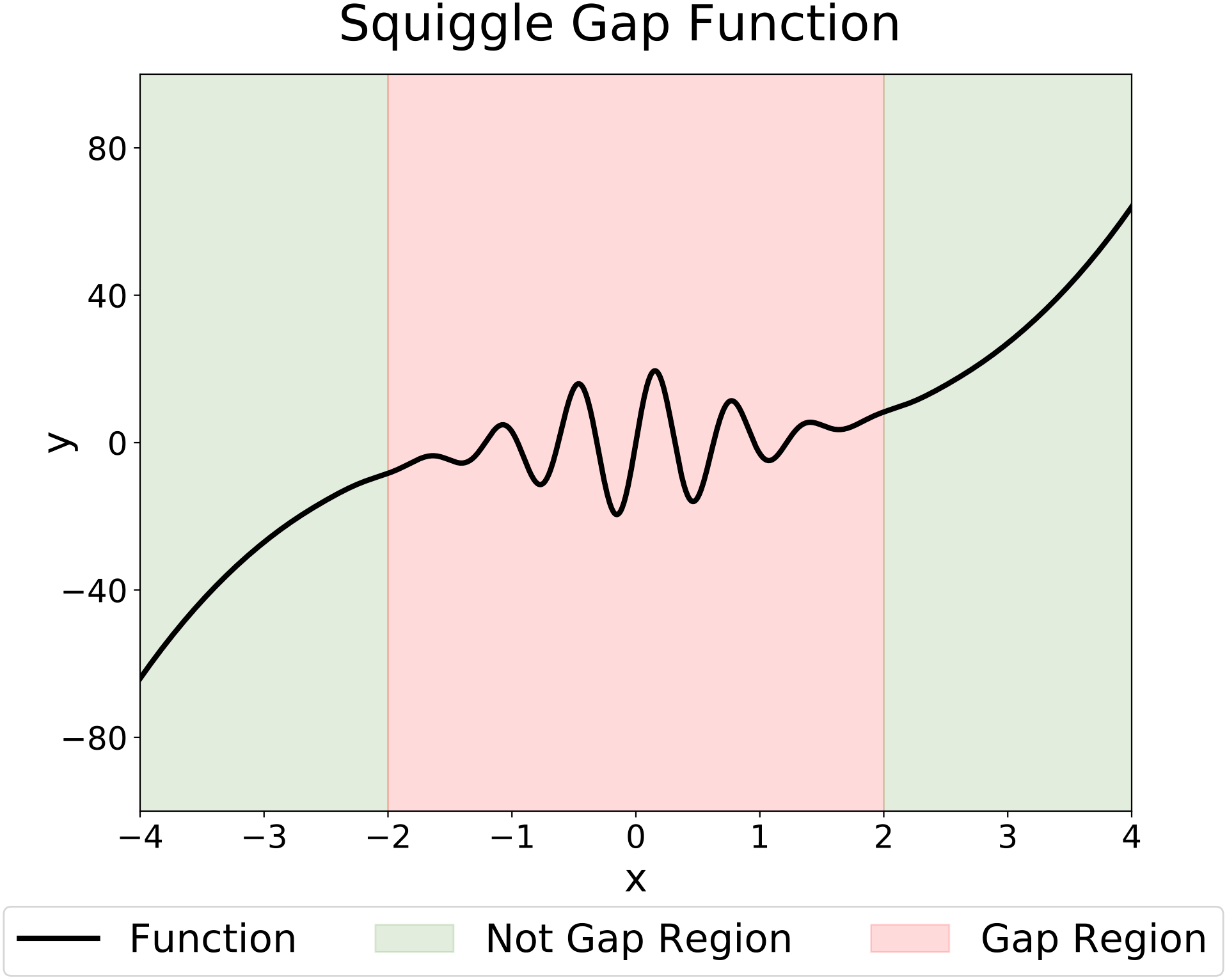}
    \caption{The ``Squiggle Gap'' data for the Transfer Learning Task: $y =x^3 + 20\exp(-x^2) \cdot \sin(10x)$.}
    \label{fig:squiggle}
  \end{subfigure}
  ~
  \begin{subfigure}{0.48\textwidth}
    \centering
    \includegraphics[width=1.0\linewidth]{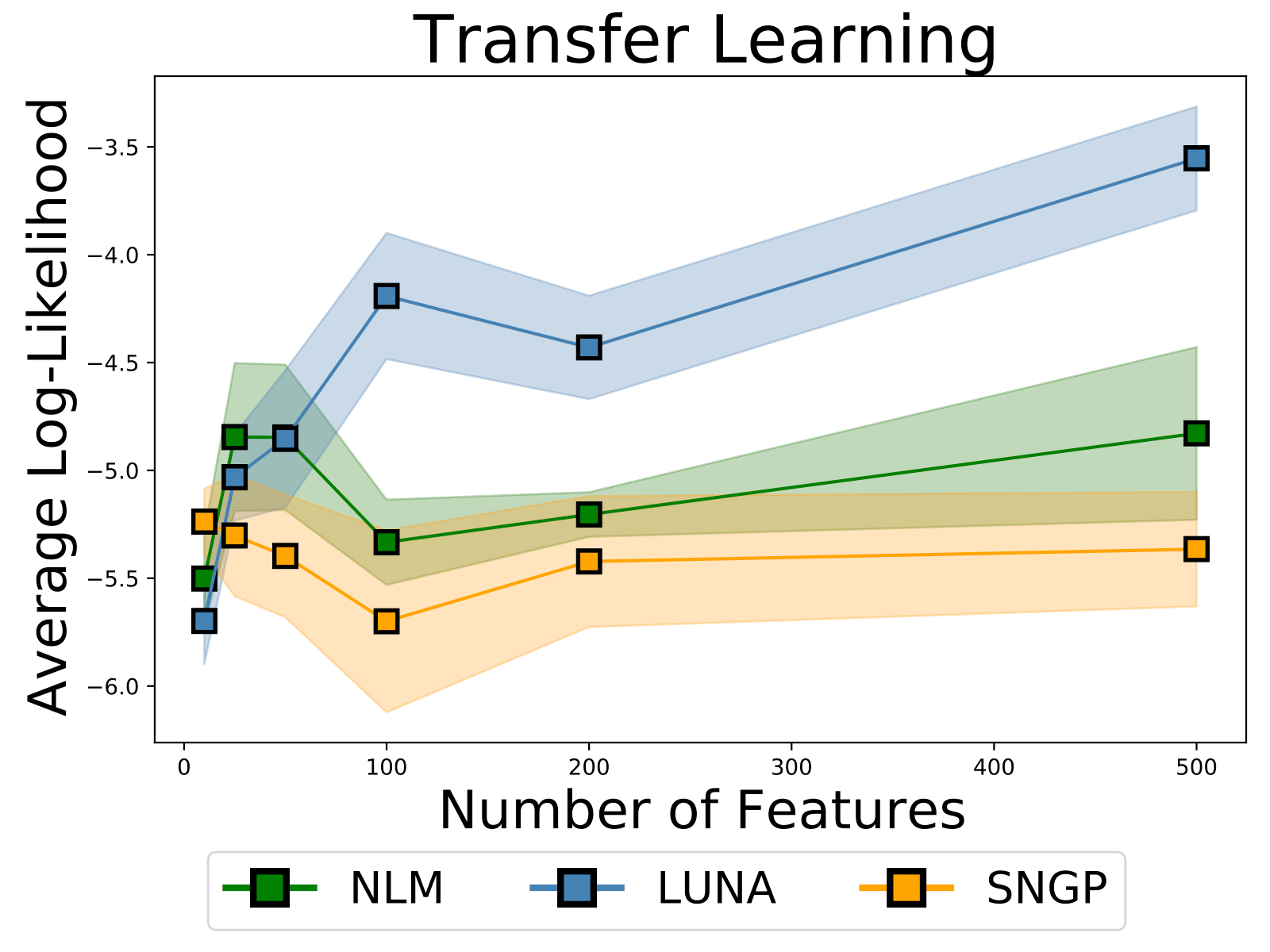}
    \caption{Transfer Learning Task}
    \label{fig:ll_vs_features_tf_relearn}
   \end{subfigure}
   
   \caption{\textbf{LUNA out-perform traditional NLM training and SNGP on transfer learning task, regardless of how many features are given.} We use the ``Squiggle Gap'' data-set (left), described in Appendix \ref{sec:synthetic-data}. On the right, we see that regardless of the number of features, LUNA outperforms NLM and SNGP under covariate shift when the neural network feature extractor is fixed but the Bayesian model is retrained on data from the gap.}
   \label{fig:squiggle-gap-tasks}
\end{figure}

\subsection{Application: LUNA's predictive uncertainty detects sampling bias in image data.} 
We show that LUNA's predictive uncertainty is sensitive to out-of-distribution points, and is therefore capable of detecting sampling bias in an image data-set.
For this task, we use LUNA trained NLMs with a ResNet18 architecture to perform age regression on the Wikipedia faces dataset, containing 62,328 facial images of actors (details in Appendix \ref{sec:training}).

In our first experiment, we trained LUNA on 26,375 faces of only male actors and test on 10,918 male (in-distribution) and 10,918 female (out-of-distribution) faces. 
We obtain a Mean Absolute Error (MAE) comparable to that of a vanilla ResNet18 ($9.52$ on training data, and $10.22$ and $11.78$ on in-distribution and out-of-distribution test-data, respectively). At the same time, the epistemic uncertainty is on average $168\%$ higher on the out-of-distribution test data than on training data (in comparison to $14\%$ higher on the in-distribution test data relative to on training data). In our second experiment, we trained LUNA on 28,271 faces of individuals younger than 30 or older than 40, while testing on 9,424 in-distribution faces and 10,376 faces of individuals between the ages of 30 and 40 (out-of-distribution). Here, we again see higher average epistemic uncertainty on out-of-distribution test data (27\% increase) than on in-distribution test data (2.02\% increase).  

In both of these experiments, we show that the predictive uncertainty provided by LUNA trained models can be used to identify test data from underrepresented sub-populations; predictions for such out-of-distribution test data can then be deferred to human experts. This task also shows that LUNA can leverage structured data more easily than GPs by using task-appropriate network architecture.

\section{Discussion}

\paragraph{Traditional NLM inference hinders learning models that express in-between uncertainty.} 
Traditional NLM inference methods learn feature bases that span a limited class of functions under the prior predictive distribution, and hence the NLM posterior predictive uncertainty does not distinguish data-poor regions from data rich-ones. We identify the cause of the problem: MLE training does not explicitly encourage the feature bases to span diverse functions away from training data, and MAP training discourages the feature bases from spanning a diverse function class. 
Moreover, we show that this problem cannot be solved by maximizing the marginal log-likelihood. 

\paragraph{UNA is a general NLM training framework for encoding for task-specific desiderata.}
We present a novel framework for training feature bases for NLMs. Our framework leverages the insight that it is difficult to train NLMs with task-specific desiderata by directly specifying an appropriate objective on the feature basis; instead, it is easier to work in ``function-space'', by designing an objective for functions spanned by the basis. 
This objective therefore indirectly encourages the feature basis to satisfy our desiderata.

We then propose an instantiation of this framework, LUNA, which provides a general-purpose method to obtain uncertainties that are calibrated in data-rich regions and higher in data-poor ones.
On both real and synthetic data, we demonstrate that LUNA training produces models with high-quality predictions on in-distribution inputs, while remaining uncertain on out-of-distribution inputs. In comparison, nearly all baselines fail to consistently remain uncertain on these out-of-distribution inputs.
We further show that LUNA outperforms baselines on a number of downstream tasks that require capturing in-between uncertainty.
Lastly, we show that LUNA can easily be scaled and adapted to work on architectures (e.g. ResNet), and that LUNA's predictive uncertainty can detect sampling bias in image data. 


\paragraph{Future work.} LUNA bases span functions that both fit the observed data and generalize differently in data-poor regions; these bases are useful for tasks in which it would suffice to use any model with predictive uncertainty that can distinguish between data-sparse and data-rich regions.
However, for some tasks, one needs to additionally incorporate domain knowledge, e.g. a length-scale or amplitude for the function likely under the prior. In future work, we hope to develop additional instantiations of UNA that can encode for other types of domain knowledge. 

In addition to proposing other instantiations of UNA, in future work we also hope to draw a broader theoretical connection between NLM training and hyper-parameter selection.
That is, one can regard the feature map of an NLM as a hyper-parameter of the Bayesian linear regression that is trained on top of it. Thus, our work shows that there is a general need for \emph{uncertainty-aware} frameworks for hyper-parameter selection as alternatives to likelihood-based selection. 

\section{Conclusion}
In this paper, we show that while NLMs are scalable and easy to implement, when trained with traditional inference, they cannot be naively deployed in risk-sensitive applications, since traditional NLM inference methods yield models that are overly certain on data from data-scarce regions of the input-space. We then propose a novel training framework, UNA, as well as an instantiation of this framework, LUNA, that mitigates this issue. LUNA training explicitly encourages for functional diversity under the NLM's prior predictive, and as a results learns models that are more uncertain on out-of-distribution inputs than on in-distribution inputs.

\acks{ST, CL, and WP are supported by the Harvard Institute of
Applied Computational Sciences. YY acknowledges support from NIH 5T32LM012411-04 and from IBM Research.}

\vskip 0.2in
\bibliography{the}


\newpage

\appendix

\addcontentsline{toc}{section}{Appendix} 
\part{Appendix} 
\parttoc 

\section{Neural Linear Model Details}
\label{sec:models}

\subsection{Posterior Predictive}

An NLM uses a neural network to parameterize basis functions for a Bayesian linear regression model by treating the output weights of the network probabilistically, while treating the rest of the network's parameters $\theta$ as hyper-parameters. 
Using notation from Section \ref{sec:background} and following standard Bayesian linear regression analysis, we can derive the posterior predictive,
\begin{align}
    p(y_\star|\mathbf{x}_\star, \mathcal{D}) &= \mathcal{N}(y_\star; \mathbf{w}_N^\intercal\phi_\theta(\mathbf{x}_\star), \sigma^2 + \phi_\theta(\mathbf{x}_\star)^\intercal\mathbf{V}_N\phi_\theta(\mathbf{x}_\star)),
\end{align}
where
\begin{equation}
    \begin{split}
        \mathbf{w}_N &= \frac{1}{\sigma^2}\mathbf{V}_N\mathbf{\Phi_\theta}^\intercal\mathbf{y}\\
        \mathbf{V}_N^{-1} &= \frac{1}{\alpha}\mathbf{I}_{M\times M} + \frac{1}{\sigma^2}\mathbf{\Phi_\theta}^\intercal\mathbf{\Phi_\theta}.
    \end{split}
\label{eq:nlm_posterior}
\end{equation}

\subsection{MAP Training}

For the MAP-trained NLM, we maximize the following objective:
\begin{equation}
    \begin{split}
        \mathcal{L}_{\mathrm{MAP}}(\theta_\mathrm{Full}) &= \log{\mathcal{N}\left(\mathbf{y}; \mathbf{\Phi_\theta w}, \sigma^2\mathbf{I}\right)} - \gamma\norm{\theta_\mathrm{Full}}_2^2\\
&=-\frac{N}{2}\log{2\pi\sigma^2} - \frac{1}{2\sigma^2}\norm{\mathbf{y}-\mathbf{\Phi_\theta w}}_2^2 -\gamma\norm{\theta_\mathrm{Full}}_2^2,
    \end{split}
\end{equation}
where $\theta_\mathrm{Full}$ represents the parameters of the full network (including the output weights). We would then extract $\theta$ from $\theta_\mathrm{Full}$ and perform the Bayesian linear regression as above.

\subsection{Marginal Likelihood Training}
\label{sec:marginal_ll}

We optimize $\theta$ to maximize the evidence or log marginal likelihood of the data by integrating out $\mathbf{w}$. For training stability and identifiability, we further regularize $\theta$ as done by ~\cite{Rasmussen}. The full objective is hence:
\begin{equation}
    \begin{split}
        \mathcal{L}_{\mathrm{Marginal}}(\theta_\mathrm{Full}) &= \log{\int p(\mathbf{y}|\mathbf{X}, \mathbf{w})p(\mathbf{w})d\mathbf{w}}\\
&=-\frac{N}{2}\log{2\pi\sigma^2} - \frac{1}{2\sigma^2}\norm{\mathbf{y}-\mathbf{\Phi_\theta w}}_2^2
-\frac{M}{2}\log{\alpha}-\frac{1}{2\alpha}\norm{\mathbf{w}_N}_2^2
-\frac{1}{2}\log{|\mathbf{V}_N|}.
    \end{split}
\end{equation}
\cite{Rasmussen} note that the addition of a regularization term $\gamma\norm{\theta}_2^2$ to $\mathcal{L}_{\mathrm{Marginal}}$ is necessary in order to estimate the observation noise variance, which otherwise tends towards zero. In Proposition \ref{thm:marg_blowup}, we also show that without this regularization term, the features $\mathbf{\Phi_\theta}$ experience pathological blow-up for ReLU networks, since large $\mathbf{\Phi_\theta}$ reduce the magnitude of the posterior mean $\mathbf{w}_N$, and hence increase $\mathcal{L}_{\mathrm{Marginal}}$.

\ThmMarginalBlowup*

\begin{proof}
Let us first establish the relationship between $\mathbf{w}_N$ and $\mathbf{\Phi_\theta}$ in the asymptotic case $\norm{\mathbf{\Phi_\theta}}\to\infty$.
From Equation \ref{eq:nlm_posterior}, 
\begin{align*}
    \frac{1}{\sigma^2}\mathbf{\Phi_\theta}^\intercal\mathbf{\Phi_\theta} \gg \frac{1}{\alpha}\mathbf{I}_{M\times M} 
    &\implies \mathbf{V}_N^{-1}\to\frac{1}{\sigma^2}\mathbf{\Phi_\theta}^\intercal\mathbf{\Phi_\theta}\\
    &\implies \mathbf{V}_N\to\sigma^2\left(\mathbf{\Phi_\theta}^\intercal\mathbf{\Phi_\theta}\right)^{-1}
\end{align*}
Hence,
$$\mathbf{w}_N\to\left(\mathbf{\Phi_\theta}^\intercal\mathbf{\Phi_\theta}\right)^{-1}\mathbf{\Phi_\theta}^\intercal\mathbf{y}\implies\norm{\mathbf{w}_N}\thicksim\frac{1}{\norm{\mathbf{\Phi_\theta}}}.$$
Note that the loss $\| \mathbf{y} -\mathbf{\Phi_\theta}\mathbf{w}\|$ is equal to $\| \mathbf{y} -\mathbf{\Phi_{\theta^c}}\mathbf{w}^c\|$, since $\mathbf{\Phi_{\theta^c}}$ is $\mathbf{\Phi_{\theta}}$ scaled by $c$ and this scaling is canceled by $\mathbf{w}^c = \frac{1}{c} \mathbf{w}$. 
Thus, since $\|\mathbf{w}^c_N\| < \|\mathbf{w}_N\|$, we have that $\mathcal{L}_{\mathrm{Marginal}}(\theta_\mathrm{Full}) < \mathcal{L}_{\mathrm{Marginal}}(\theta^c_\mathrm{Full})$.

\end{proof}
The above proposition tells us that that we can continue to increase $\mathcal{L}_{\mathrm{Marginal}}$ by reducing $\norm{\mathbf{w}_N}$. Hence, if we do not regularize $\theta$, the training will continually increase $\norm{\mathbf{\Phi_\theta}}$ to affect a decrease in $\norm{\mathbf{w}_N}$.

The addition of the regularization term $\gamma\norm{\theta}_2^2$ to $\mathcal{L}_{\mathrm{Marginal}}$, however, biases training towards inexpressive feature bases for the same reason we identified in Section \ref{sec:nlm_pathologies}. In Figure \ref{fig:marg_highgamma}, we show that with regularization, the feature bases learned by optimizing $\mathcal{L}_{\mathrm{Marginal}}$ are inexpressive. In Figure \ref{fig:marg_lowgamma} we see that even with $\gamma$ set close to zero, the learned feature bases are not consistently expressive across random restarts.

\section{LUNA's Diversity Penalty}
\label{luna-training-objective}

We adopted the diversity penalty in LUNA's objective from the work of \cite{diversity_enforcement}.
We use the cosine similarity function on the gradients of the auxiliary regressors:
\begin{equation*}
    \begin{split}
        \text{CosSim}^2&\left(\nabla_{\mathbf{x}} f_i(\mathbf{x}), \nabla_{\mathbf{x}} f_j(\mathbf{x}) \right) = 
        \frac{\left(\nabla_{\mathbf{x}} f_i(\mathbf{x})^\intercal \nabla_{\mathbf{x}} f_j(\mathbf{x})\right)^2}
        {\left(\nabla_{\mathbf{x}} f_i(\mathbf{x})^\intercal  \nabla_{\mathbf{x}} f_i(\mathbf{x})  \right) \left( \nabla_{\mathbf{x}} f_j(\mathbf{x})^\intercal \nabla_{\mathbf{x}} f_j(\mathbf{x})\right)}
    \end{split}
\end{equation*}
This acts as a measure of orthogonality, equal to one when the two inputs are parallel, and 0 when they are orthogonal.
A higher penalty, $\lambda$, in the training objective penalizes parallel components, hence enforcing diversity.

In practice, these gradients can be computed using a finite differences approximation.
That is, we approximate gradients as:
\begin{equation*}
    \frac{\partial f_i(\mathbf{x})}{\partial x_d} \approx \frac{f_i\left(\mathbf{x} + \delta\mathbf{x}_d\right) - f_i\left(\mathbf{x}\right)}{\delta x_d}
\end{equation*}
where $\delta\mathbf{x}_d = [0,\hdots,0,\delta x_d,0, \hdots,0]^\intercal$ represents a D-dimensional vector of zeros with a small perturbation in the $d^\text{th}$ dimension. 
We sample these perturbations according to $\delta x_d \sim \mathcal{N}(0, \epsilon)$, where $\epsilon$ can be set using the range of the data.

In our experiments, we anneal the weight of the diversity penalty using a number of different schedules. We also scaled the diversity penalty by a factor $C = 2B/(M(M-1))$, where $B$ is batch size, and $M$ is the number of auxiliary regressors, so that it carried the same weight across different values of $B$ and $M$.
The three annealing schedules we tested were $f_\text{sqrt}(t) = C\sqrt{t/T}$, $f_\text{sigmoid} = C/(1 + \exp(-6t/T + 3))$, and $f_\text{tanh}(t) = C(\tanh(6t/T - 3) + 1)/2$, where $T$ is the number of epochs.
We use the model selection schemed described in Section \ref{sec:luna} to select an annealing schedule.

\section{Experimental Setup} \label{sec:exp-setup}

\subsection{Baseline Methods}

\paragraph{Linearised Laplace (Lin Lap)~\cite{uci_gap}.} Linearised Laplace is an approximate inference method for BNNs that finds a mode of the BNN posterior, and then fits a Gaussian to that mode. 
As noted by \cite{uci_gap}, this inference method scales cubically with the number of parameters in the model.
Additionally, \cite{uci_gap} note that Linearised Laplace does not work with all activation functions, e.g. ReLU (which is commonly used in standard architectures for its desirable extrapolation properties).

\paragraph{Monte Carlo Dropout (MCD)~\citep{mcd}.} MCD casts dropout training in neural networks as approximate Bayesian inference in Deep Gaussian Processes. 
While dropout during training is a common feature of modern neural network architectures, MCD maintains the dropout during testing time too. 
Using multiple stochastic forward passes through the network and averaging the results, MCD is able to obtain a predictive distribution.

\paragraph{Spectral-normalized Neural Gaussian Process (SNGP)~\citep{liu2020simple}.} SNGP is a model that was originally designed for classification, but in this work we have adapted it for regression.
It uses a distance-aware Recurrent Neural Network (RNN) to learn a feature map for the data. This feature map is then used as input for a GP.
In this work, we use a GP with an RBF kernel at the last layer and perform inference either analytically (SNGP) or with a random Fourier feature expansion (RFF SNGP). 

\paragraph{Anchored Ensembles~\citep{pearce2018bayesian}.} Anchored Ensembles is an alternative training method for neural network ensembles, in which the ensemble members are initialized using an ``anchoring'' distribution and are then regularized towards their initial parameters:
\begin{equation*}
    Loss_j = \frac{1}{N}\left\lVert\mathbf{y} - \hat{\mathbf{y}}_j\right\rVert_2^2 + \frac{1}{N} \left\lVert\mathbf{\Gamma}^{1/2}(\mathbf{\theta}_j - \mathbf{\theta}_{anc, j})\right\rVert_2^2.
\end{equation*}
The regularization matrix is defined as $\text{diag}(\mathbf{\Gamma}) = \sigma_{\epsilon}^2/\sigma_{prior_i}^2$, where $\sigma_{\epsilon}$ is the noise variance of the data, $\sigma_{prior_i}$ is the anchor variance.
That is, the anchoring parameters are sampled according to $\mathbf{\sigma} \sim \mathcal{N}\left( \mathbf{\mu}_{prior}, \mathbf{\Sigma}_{prior} \right)$.
In our case, we used one value of $\sigma_{prior}$ and always set $\mathbf{\mu}_{anchor} = \mathbf{0}$.
Additionally, we follow the original work and decouple the initial parameters from the anchoring distribution, where initial parameters are sampled according to $\theta \sim \mathcal{N}(0, \sigma_{init}^2)$.
These anchoring points ensure that the ensemble fits to the data but also maintains its original diversity from the initialization, thereby allowing it capture uncertainty in data-scarce regions.
Moreover, this training method implicitly performs Bayesian inference on a Gaussian Process as the ensemble and network sizes tend towards infinity. 

\paragraph{Bootstrap Ensembles.} We construct an ensemble of neural network regressors, each trained on a different sub-sample of the original training set. We then use mean and variance of the outputs across all ensemble members as a predictive distribution.

\paragraph{Vanilla Ensembles.}
We construct an ensemble of neural network regressors, each initialized randomly and trained independently via gradient descent. We then use mean and variance of the outputs across all ensemble members as a predictive distribution.

\subsection{Evaluation Metrics} \label{sec:eval-metrics}

We used the following evaluation metrics to compare the performance of UNA against our baselines:

\paragraph{Avg. Test Log-Likelihood (LL).}
We compute average the probability of each test point $\mathbf{x}_n^*, y_n^*$ under the model's posterior predictive:
\begin{align}
    \text{LL} &= \frac{1}{N} \sum\limits_{n=1}^N p(y_n^* | \mathbf{x}_n^*).
\end{align}
For NLM (and UNA), this is computed as:
\begin{align}
    \text{LL} &= \frac{1}{N} \sum\limits_{n=1}^N \mathbb{E}_{p(\mathbf{w} | \mathcal{D}, \theta)} \left[ \mathcal{N}\left( y_n^* | \phi_\theta(\mathbf{x}_n^*)^\intercal \mathbf{w}, \sigma^2 \mathbf{I} \right) \right].
\end{align}

\paragraph{Test Root Mean Square Error (RMSE).}
We compute the RMSE between the mean of the posterior predictive, $\hat{y}_n^*$ and the corresponding true outcome $y_n^*$ across the test set. 
\begin{align}
    \text{RMSE} &= \sqrt{\frac{1}{N} \sum\limits_{n=1}^N \left( \hat{y}_n^* - y_n^* \right)^2}.
\end{align}
For NLM (and UNA), this is computed as:
\begin{align}
    \text{RMSE} &= \sqrt{\frac{1}{N} \sum\limits_{n=1}^N \left( \mathbb{E}_{p(\mathbf{w} | \mathcal{D}, \theta)} \left[  \phi_\theta(\mathbf{x}_n^*)^\intercal \mathbf{w} \right] - y_n^* \right)^2}.
\end{align}

\paragraph{Epistemic Uncertainty (EU).}
Since in this work, we assume homoscedastic noise for all models,
we compute epistemic uncertainty as the standard deviation of the posterior predictive, without the added observation noise $\sigma^2 \mathbf{I}$:
\begin{align}
    \text{EU} &= \frac{1}{N} \sum\limits_{n=1}^N \sqrt{\mathbb{V} [y_n^* | \mathbf{x}_n^*] - \sigma^2 \mathbf{I}}.
\end{align}
For NLM (and UNA), this is computed as:
\begin{align}
    \text{EU} &= \frac{1}{N} \sum\limits_{n=1}^N \sqrt{\mathbb{V}_{p(\mathbf{w} | \mathcal{D}, \theta)} \left[  \phi_\theta(\mathbf{x}_n^*)^\intercal \mathbf{w} \right]}.
\end{align}
In this work, we evaluate epistemic uncertainty both where there is data (i.e. ``not in gap'') and where there is no data (i.e. ``in the gap''). 

\paragraph{Gap to Non-Gap Epistemic Uncertainty Relative Change.}
To check whether a model has higher epistemic uncertainty where there is no data, we compute the relative change between the epistemic uncertainty in the gap vs. not-gap regions:
\begin{align}
    \text{EURC} &= \frac{\text{EU}_\text{GAP} - \text{EU}_\text{NOT-GAP}}{\text{EU}_\text{GAP}}.
\end{align}
When EURC positive, it means that the model is more uncertain in data-sparse regions, and when it is negative, it is even more confident where there no data relative to where there is data. 

\subsection{Synthetic Data} \label{sec:synthetic-data}

\paragraph{Cubic Gap Example.}
Following ~\cite{Rasmussen}, we construct a synthetic 1-D dataset comprising 100 train and 100 test pairs $(x, y)$, where $x$ is sampled uniformly in the range $[-4,-2]\cup[2,4]$ and $y$ is generated as $y =x^3 + \epsilon, \epsilon\sim\mathcal{N}(0, 3^2)$. 

\paragraph{Squiggle Gap Example (for the Transfer Learning Task).}
We construct a synthetic 1-D dataset in which the training $x$'s are uniformly sampled in the range $[-4,-2]\cup[2,4]$
and $y =x^3 + 20\exp(-x^2) \cdot \sin(10x) + \epsilon, \epsilon\sim\mathcal{N}(0, 3^2)$.
As shown in Figure \ref{fig:squiggle}, this function is identical to the function from the Cubic Gap Example, but with unexpected variations in the gap.
For the generalization experiment, the test $x$'s were sampled from the non-gap region $[-4,-2]\cup[2,4]$, whereas for the transfer learning experiment, test $x$'s were sampled from the gap region $[-2,2]$.

\paragraph{BayesOpt Examples.}
We used common BayesOpt benchmarks to evaluate the usefulness of our uncertainties.
These benchmarks were adapted from HPOLib 1.5 \citep{eggensperger_toward} and represent a variety of tasks that are difficult or impossible for traditional optimization techniques.

First, we used the Branin function, a 2-dimensional benchmark with multiple global minima and shallow valleys between the minima.
The function is defined as:
\begin{gather*}
    f(\textbf{x}) = \left(x_2 - \frac{5.1}{4\pi^2}x_1^2 + \frac{5}{\pi}x_1 - 6\right)^2 + 10\left(1-\frac{1}{8\pi}\right)\cos \left(x_1\right) + 10
\end{gather*}
The input domain used is the square $x_1 \in [-5, 10]$ and $x_2 \in [0, 15]$.
In this domain, the global minima occur at $\textbf{x}^* = (-\pi, 12.275)$, $(\pi, 2.275)$, and $(9.42478, 2,475)$, with minimum $f(\textbf{x}^*) = 0.397887$.

The Hartmann6 function was also used and is a higher dimensional function on a small domain.
It is defined as:
\begin{equation*}
    f(\textbf{x}) = -\sum_{i=1}^{4}\alpha_i \exp\left( -\sum_{j=1}^{6} A_{ij}\left(x_j - P_{ij} \right)^2 \right)
\end{equation*}
where $\alpha = [1.0, 1.2, 3.0, 3.2]^\intercal$,
\begin{align*}
\mathbf{A} &= \begin{bmatrix}
10 & 3 & 17 & 3.5 & 1.7 & 8 \\
0.05 & 10 & 17 & 0.1 & 8 & 14 \\
3 & 3.5 & 1.7 & 10 & 17 & 8 \\
17 & 8 & 0.05 & 10 & 0.1 & 14
\end{bmatrix}, \\
\mathbf{P} &= 
10^{-4} \begin{bmatrix}
1312 & 1696 & 5569 & 124 & 8283 & 5886 \\
2329 & 4135 & 8307 & 3736 & 1004 & 9991 \\
2348 & 1451 & 3522 & 2883 & 3047 & 6650 \\
4047 & 8828 & 8732 & 5743 & 1091 & 381 
\end{bmatrix}.
\end{align*}
The input domain used is the hypercube $x_i \in (0, 1)$ for all $x_i$, with global minimum 
$f(\textbf{x}^*) = -3.32237$ at
\begin{gather*}
\textbf{x}^* = (0.20169, 0.150011, 0.476874, 0.275332, 0.311652, 0.6573).
\end{gather*}

A popular application for BayesOpt is hyper-parameter tuning.
In our case, we used each model to optimize classification models for MNIST data.
The SVM model was optimized over the regularization parameter and kernel coefficient, where the domain for each parameter is $[-10, 10]$ on the log-scale.
Due to computational complexity, the first 5,000 data points in the training set were used in this experiment. 
Similarly, the logistic benchmark is a logistic regression classifier in which the learning rate, $l_2$-regularization, batch size, and dropout ratio on inputs was tuned.
The learning rate domain is $[-6, 0]$ on the log-scale, the $l_2$-regularization domain is $[0,1]$, the batch size domain is $[20,2000]$, and the dropout ratio domain is $[0. 0.75]$.

All results for the BayesOpt examples are reported in terms of the error, $\left| f(\mathbf{x}) - f(\mathbf{x}^*) \right|$,
where classification global minimum is at 0.


\subsection{Real Data} \label{sec:uci-gap-examples}

\paragraph{UCI Gap Data.}
We used 3 standard UCI~\citep{uci_data} regression data sets and modify them to create 6 ``gap data-sets'', wherein we purposefully created a gap in the data where we can test our model's in-between uncertainty (i.e. we train our model on the non-gap data and test the model's epistemic uncertainty on the gap data).
We adapt the procedure from ~\cite{uci_gap} to convert these UCI data sets into UCI gap data sets. For a selected input dimension,
we (1) sort the data in increasing order in that dimension,
and (2) remove middle $1/3$ to create a gap.
We specifically selected input dimensions that have high correlation with the output in order to ensure that the learned model should have epistemic uncertainty in the gap;
that is, if we select a dimension that is not useful for prediction, any model need not have increased uncertainty in the gap.
The features we selected are:
\begin{itemize}
    \item Boston Housing: ``Rooms per Dwelling'' (RM), ``Percentage Lower Status of the Population'' (LSTAT), and ``Parent Teacher Ratio'' (PTRATIO)
    \item Concrete Compressive Strength: ``Cement'' and ``Superplasticizer''
    \item Yacht Hydrodynamics: ``Froude Number''
\end{itemize}
The not gap region of the data was then split into 10 different 80-10-10 train,test, validation splits.
Final results are computed as the mean and standard deviation of the predictions over all of the splits.

\paragraph{Standard UCI Data.}
We used 6 UCI regression data sets to benchmark our models.
Those are the Boston Housing, Concrete Compressive Strength, Yacht Hydrodynamics, Energy Efficiency, Abalone (Kin8nm), and Red Wine Quality datasets.
For each experiment, we split the data into 90\% train data and 10\% test data.
We then used 20\% of the training data for the validation set.

\subsection{Architecture} \label{sec:training}


\paragraph{Cubic Gap Experiments.} 
We used 2-layer 50 hidden units ReLU networks for all neural network based architectures (except for the BNN with HMC, in which a 1-layer ReLU network was used to reduce computational cost, and except for the BNN with Linearised Laplace, in which we used Tanh activation since ReLU did not work \citep{uci_gap}, and except  RFF SNGP, in which we again used Tanh activation since ReLU did not work). 
We used $M=20$ auxiliary regressors in LUNA, and 50 independent networks for the bootstrap ensemble. 
We used a GP with a kernel that is the sum of a Matern-$5/2$ kernel with length scale $1.0$ and a white kernel with ground truth noise variance $\sigma^2=3^2$.
For SNGP, we used an RBF kernel with length scale of $1.0$, a normalization factor of 20, and dropout rate of 0. 
We trained it using SGD with 1000 epochs, a step size of $10^{-1}$, and no mini-batching.
For the RFF SNGP, we used a scaling coefficient of $10^{-8}$, normalizing factor of $10.0$, GP layer width of $500$, regularization parameter of $10^{-6}$ and dropout rate of 0. 
We again trained it using SGD with 50 epochs, a step size of $10^{-3}$, and a mini-batch size of 10. 
We found that training longer, or with higher scaling coefficient led to instability in the posterior covariance prediction.
We trained an Anchored Ensemble consisting of 5 networks, using an anchor mean of 0, anchor variance of 10, data variance of 3, and alpha of 0.5. 
We trained using SGD with 500 epochs, a step size of $2 \times 10^{-5}$, and a mini-batch size of 32.


\paragraph{Squiggle Gap Experiments.}
We use 2-layer ReLU networks with 50 units in the first hidden layer and variable units in the second hidden layer, depending on the experiment. LUNA used $M=20$ auxiliary regressors.
With the RFF SNGP model, we use a 2-layer 50 unit model with Tanh activation, regularization of $10^{-6}$, a dropout rate of 0, normalization factor of 10, and a scaling coefficient of $10^{-6}$.



\paragraph{BayesOpt Experiments.}
For the BayesOpt experiments, the neural networks had 3 hidden layers with 50 ReLU units in each layer,  following the setup of \cite{Snoek}.
LUNA used $M=50$ auxiliary regressors.
RFF SNGP used a GP width of 200 units and and Tanh activation. SNGP used an RBF kernel with length scale of 1.0.

\paragraph{Detecting Sampling Bias Experiments.}
For the sampling bias in data detection experiment, the images were first put through a pre-trained ResNet18 to obtain 1000-dimensional features. A single hidden layer neural network with 500 ReLU units was then fit on these features using LUNA training. For the experiment where female faces were held out from the train set, reasonable values of $\log{\gamma} = -1$, $\log{\lambda} = -2$ and $\log{\alpha} = 6$ were used without any tuning. For the experiment where faces of ages 30-40 were held out from the train set, reasonable values of $\log{\gamma} = -1$, $\log{\lambda} = 2$ and $\log{\alpha} = 6$ were used without any tuning. Training for both experiments were done for 10 epochs using a mini-batch size of 1024.

\subsection{Hyper-parameters}\label{sec:hyperparams}

We selected hyper-parameters for all models using $20\%$ of the training data as a held-out validation set. For all baselines (except for LUNA, which has its own model selection criteria), we selected hyper-parameters based on validation log-likelihood. 

\paragraph{Synthetic Data}
The ground truth noise variance $\sigma^2$ was used in all models.

For NLM, the regularization hyper-parameter $\gamma$ was selected by maximizing validation log-likelihood using 50 iterations of Bayesian optimization over the range $\gamma\in[10^{-2}, 10^2]$, initialized with 10 iterations of random search.

For LUNA, the regularization hyper-parameter $\gamma$ and the diversity hyper-parameter $\lambda$ were selected using using 50 iterations of Bayesian optimization over the range $\gamma\in[10^{-2}, 10^2]$ and $\lambda\in[10^{-2}, 10^2]$, initialized with 10 iterations of random search.

For MCD, the dropout rate $p$ was selected using 50 iterations of Bayesian optimization over the range $p\in[10^{-3}, 0.5]$, initialized with 10 iterations of random search.
The model precision $\tau$ was set to the inverse of the ground truth noise variance $\sigma^2$. 
A reasonable $1000$ forward passes were used to obtain model uncertainty.

For Anchored Ensembles, we hand tuned the hyper-parameters to obtain good in-between uncertainty.
We used hyper-parameters an anchor mean of 0, an anchor variance of 10, a data variance of 3, and an initial variance of 0.5.

For bootstrap ensembles, we used 50 neural nets with no regularization.

For SNGP, we hand tuned hyper-parameters to obtain better in-between uncertainty, we used $c = 20$, and no dropout.

For RFF SNGP, we hand tuned hyper-parameters to obtain better in-between uncertainty. 
We used Tanh activation, a GP layer width of 500, $m = 10^{-8}$, $\lambda = 10^{-6}$, $c = 10$.

\paragraph{Real Data}
We first fit a maximum a posteriori (MAP) model to the data sets and use the variance of the output errors as the noise variance $\sigma^2$ in all models.

For NLM, the regularization hyper-parameter $\gamma$ was selected by maximizing validation log-likelihood using 50 iterations of Bayesian optimization over the range $\gamma\in[10^{-2}, 10^2]$, initialized with 10 iterations of random search.

For LUNA, the regularization hyper-parameter $\gamma$ and the diversity hyper-parameter $\lambda$ were selected using using 50 iterations of Bayesian optimization over the range $\gamma\in[10^{-2}, 10^2]$ and $\lambda\in[10^{-2}, 10^2]$, initialized with 10 iterations of random search.


For MCD, the dropout rate $p$ was selected between $p\in\{0.005, 0.05\}$ by maximizing validation log-likelihood, following the specification in \citep{mcd}.
Since the other models used the same learned noise variance $\sigma^2$, we set the model precision $\tau$ to the inverse of this $\sigma^2$. 
A reasonable $1000$ forward passes were used to obtain model uncertainty.

For Anchored Ensembles, 5 networks with anchor mean 0 were used in each case.
Initially, a grid search in log space was used, where anchor variance, data variance, and initial variance were searched in $\{10^{-2}, 10^{-1}, 10^0, 10^1\}$.
Training was done with SGD for 50,000 epochs and a learning rate of $10^{-9}$ and no batches.
This learning rate was chosen help with numerical stability.
This training resulted in poor fit to the data.
After hand-tuning, we arrived at an anchor variance of 1.0, data variance of 1.0, and initial variance of 0.5, trained with SGD for 10,000 epochs, a step size of $10^{-5}$, and no batches.

For Bootstrapped Ensembles, 25 networks were used with no regularization.

For SNGP, we had to hand-tune the parameters.
We used the RBF kernel with length scale 1.0, $c = 5$ and no dropout.
For training, we used SGD with 5000 epochs, a step size of $10^{-2}$, and no batches.

For RFF SNGP, after searching over $c \in \{ 1, 5, 10 \}$, dropout rate of $\{0.01, 0.05, 0.2  \}$, $m=10^{-8}$ with ReLU activation, we found the model did not fit the data.
We then used Tanh activation and a GP layer width of 200, scaling coefficient of $10^{-4}$, normalizing factor of 10.0, regularization of $10^{-6}$, and dropout rate of 0 for all datasets. It was trained using SGD for 100 epochs with step size of $10^{-3}$ and batch size of 10.

\paragraph{BayesOpt} BayesOpt hyper-parameters were generally chosen with a grid search in log-space.
Each run was initialized with different randomly selected points.
We used 5 points for the Branin function, 10 points for the Hartmann6 function, 3 points for the SVM function, and 3 points for the logistic function.
Additionally, early stopping was found to both improve results and significantly reduce training time for many of the models.
Note: many of the models achieved near-optimal performance after very little tuning. Because of this, only those values were tested.


For LUNA, on all functions, we searched over $\alpha \in \{10^{-1}, 10^0, 10^1, 10^2 \}$, $\gamma \in \{10^{-2}, 10^{-1}, 10^0, 10^1\}$ and $\lambda \in \{10^{-2}, 10^{-1}, 10^0, 10^1\}$.

For NLM, on all functions, we searched over $\alpha \in \{1, 5, 10 \}$ and $\lambda \in \{10^{-4}, 10^{-3}, 10^{-2} \}$.

For MCD, on all functions, we searched over $p \in \{0.05, 0.1, 0.2, 0.5 \}$, $\tau \in \{10^{-4}, 10^{-3}, 10^{-2} \}$, with 50 forward passes.

For anchored ensembles, we searched for anchor variance, data variance, and initial variance.
Initially, the anchor variance search domain was $\{10, 50, 100\}$, the data variance search domain was $\{0.1, 1.0\}$, and initial variance search domain was $\{10, 50 \}$ for the Branin and Hartmann6 functions.

For bootstrapped ensembles, we searched the regularization parameter over $\lambda \in \{0, 10^{-4}, 10^{-3}, 10^{-2}, 10^{-1}\}$.

For the GP, we used an RBF kernel and performed grid search for length scale.
For all benchmarks the domain used was $\{0.1, 1.0, 5.0, 10.0\}$

For SNGP, we used an RBF kernel with length scale 1, and for all functions, used $c \in \{1, 5, 10 \}$ and dropout rate was $\{0.01, 0.1, 0.2 \}$

For RFF SNGP, we used $c \in \{1, 5, 10 \}$, dropout rate was $\{0.01, 0.1, 0.2 \}$, $\lambda \in \{10^{-4}, 10^{-3}, 10^{-2} \}$. We set $m = 10^{-8}$.

\FloatBarrier
\section{Experimental Results}
\subsection{Qualitative Evaluation of NLM Predictive Uncertainty Across Random Restarts}\label{sec:appendix_toy_results}


In this section we present a qualitative evaluation of the NLM predictive uncertainty, specifically examining the effect of regularization as well as the depth and width of the feature basis across random restarts.
We demonstrate that the NLM training objective, as justified by our theoretical analysis, either explicitly discourages expressing or is not consistently able to express in-between uncertainty. 

\paragraph{Effect of Regularization on NLM Posterior Predictive Uncertainty.}
In Figure \ref{fig:rr_reg}, we show NLM prior predictive samples and posterior predictive distributions, with and without regularization, across random restarts.
With regularization, NLM is unable to get expressive prior predictive samples and hence increased in-between posterior predictive uncertainty. 
While this can be achieved with no regularization, since the training objective does not explicitly encourage for diversity in the prior predictive, the posterior predictive does not consistently express in-between uncertainty across random restarts.
 
While in low dimensional data we can visually select for the random restarts for which the NLM prior predictive is diverse, on a real high-dimensional data, we cannot. 
Moreover, on such data we do not know where the data-scarce regions are and thus where to expect in-between uncertainty, since we have shown in Section \ref{sec:nlm_pathologies} that we cannot use log-likelihood to evaluate uncertainty. 

\paragraph{Effect of the NLM Basis Dimensionality on NLM Posterior Predictive Uncertainty.}
In Figure \ref{fig:rr_features}, we show the NLM posterior predictive distributions across a varying numbers of basis features $L$ and across random restarts, and in Figure \ref{fig:rr_depth} we show the NLM posterior predictive across a varying numbers of hidden layers and across random restarts. We see that increasing number of features is more important than increasing depth for expressing posterior predictive in-between uncertainty.

\paragraph{Effect of Regularization and Prior Variances on Marginal Likelihood NLM Training on Posterior Predictive Uncertainty.}
We show NLM prior predictive samples and posterior predictive distributions for marginal data likelihood training (described in Appendix \ref{sec:marginal_ll}). This is shown for different regularization and prior variances across random restarts in Figures \ref{fig:marg_lowgamma} and \ref{fig:marg_highgamma}. We see the same trends here as what was observed above with traditional MAP training: (1) with high regularization (high $\gamma$ or low $\alpha$), the feature bases for the learned $\theta$ do not span diverse functions in the prior predictive and hence cannot capture in-between uncertainty, and (2) with low regularization (low $\gamma$ or high $\alpha$), the model has the potential to capture in-between uncertainty, but it does so inconsistently across random restarts and needs to rely on good initialization.

\begin{figure*}
    \centering
    \includegraphics[width=\textwidth]{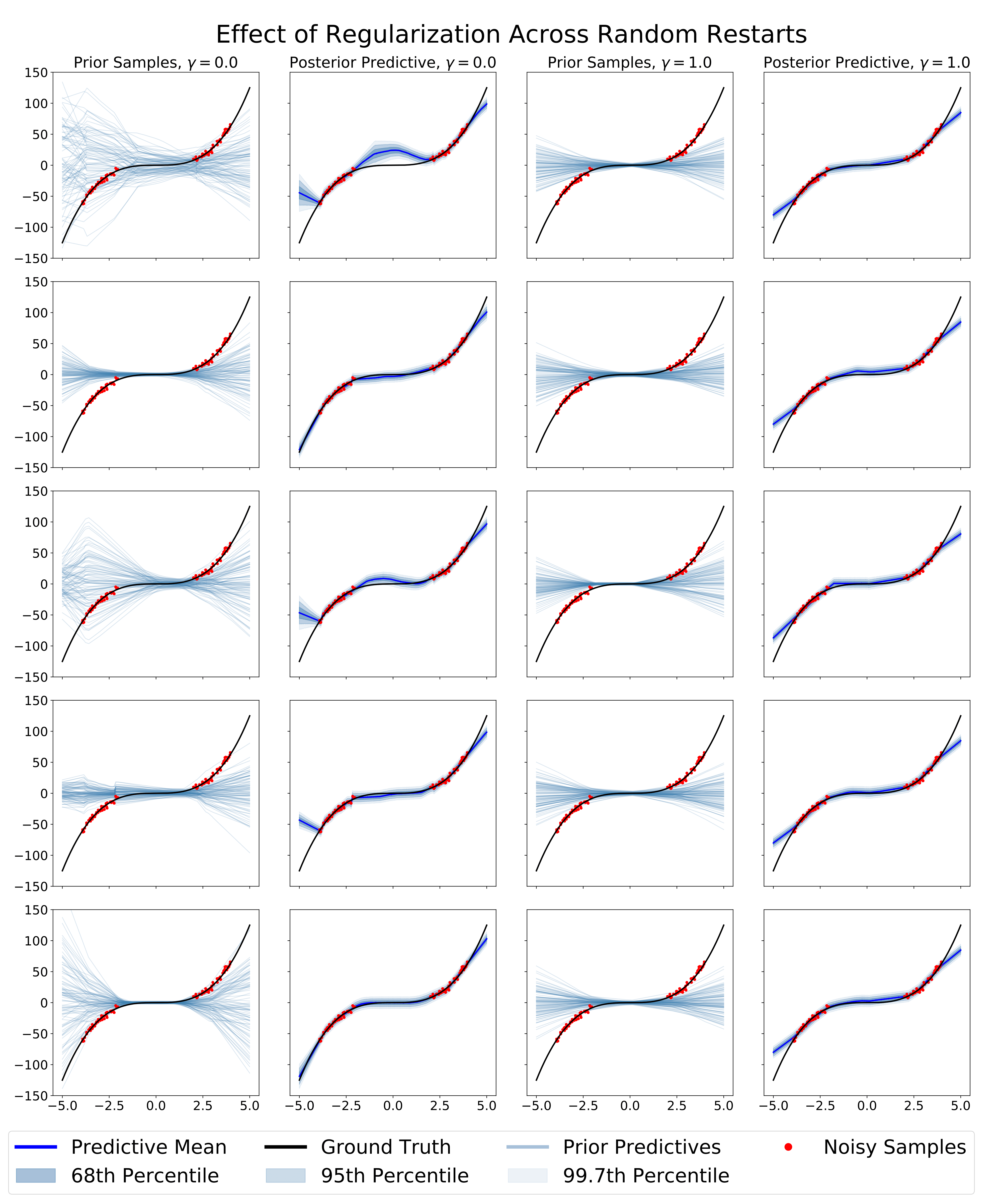}
    \caption{\textbf{MAP training discourages functional diversity under the prior predictive, and therefore cannot capture in-between uncertainty.} We used MAP training to train an NLM with a 2-layer ReLU network with 50 and 20 neurons in the first and second layers respectively (i.e. 20 features). With no regularization and very noisy priors, the NLM models in-between uncertainty, albeit inconsistently. With regularization, we see the priors are not expressive enough and the NLM fails to ever capture in-between uncertainty.}
    \label{fig:rr_reg}
\end{figure*}

\begin{figure*}
    \centering
    \includegraphics[width=\textwidth]{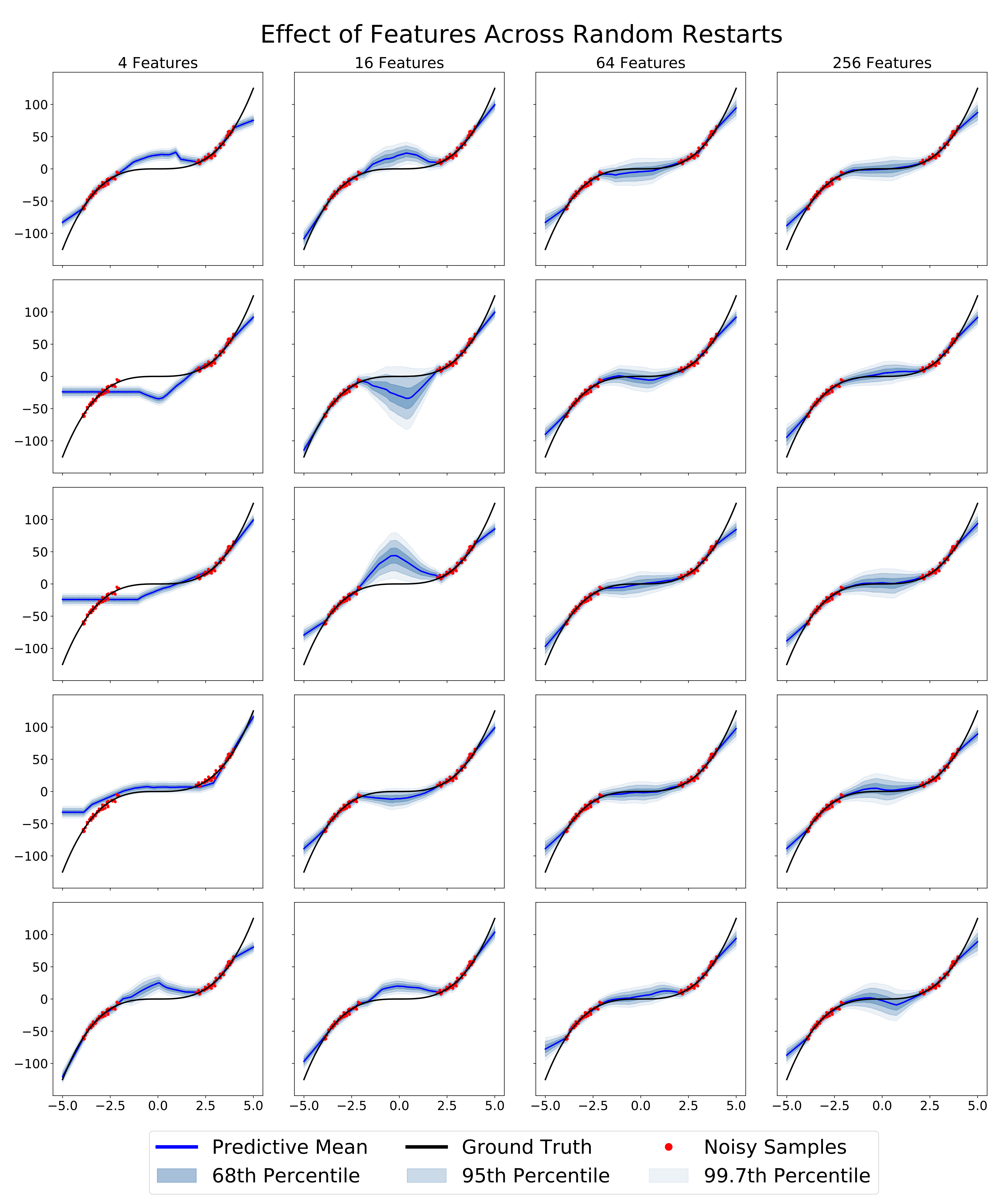}
    \caption{\textbf{Increasing the number of features of the NLM does not help MLE training capture in-between uncertainty consistent.} We used MAP training to train an NLM with a 2-layer ReLU network with 50 neurons in the first layer, varying the number of features in the second layer. We see clearly as model capacity increases, the NLM better fits the data better. However, this increased capacity still fails to consistently model in-between uncertainty across random restarts.}
    \label{fig:rr_features}
\end{figure*}

\begin{figure*}
    \centering
    \includegraphics[width=\textwidth]{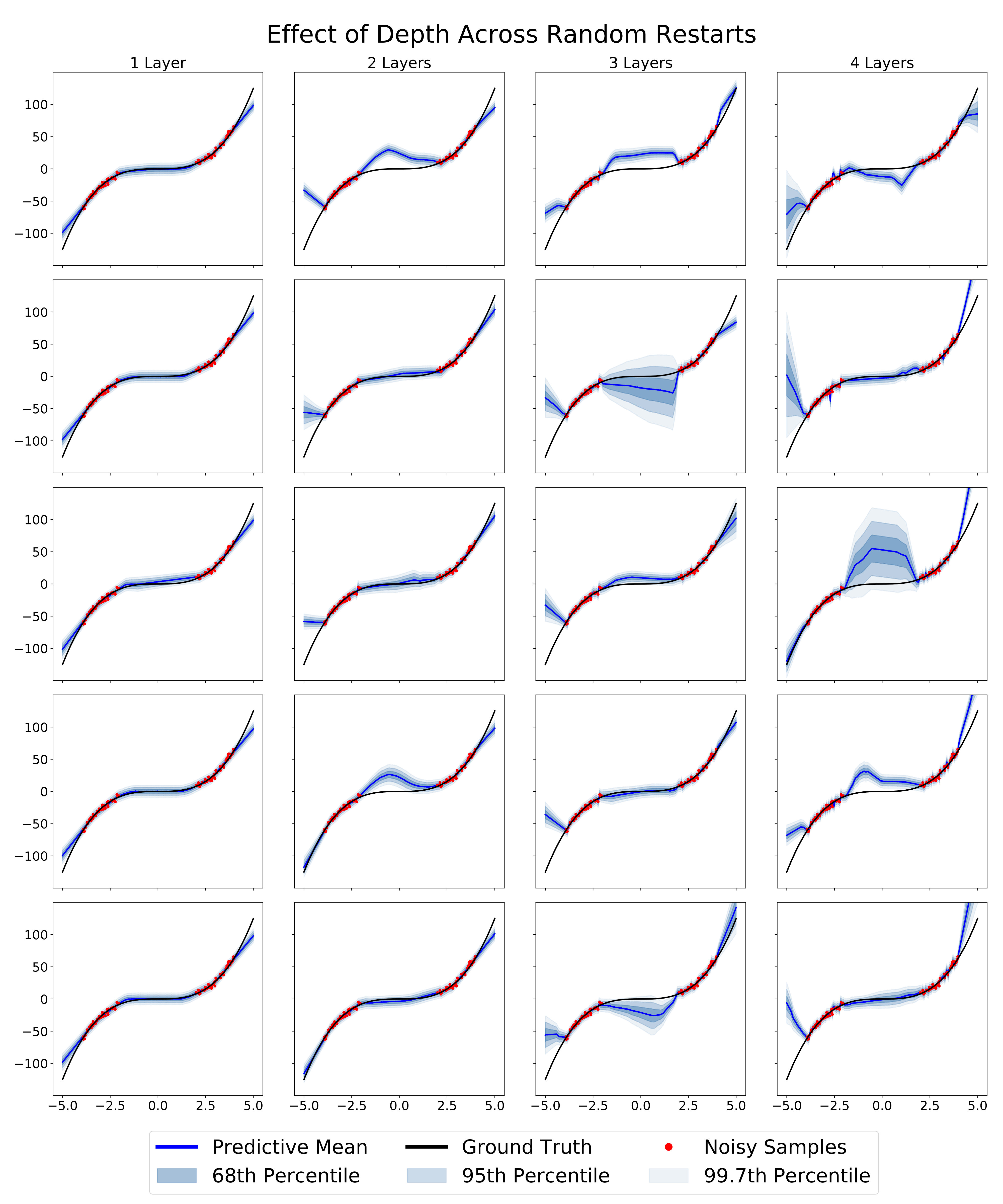}
    \caption{\textbf{Increasing the depth of an NLM does not help MAP training capture in-between uncertainty consistently.} We used MAP training to train an NLM with varying depth. We see that NLM is able to capture more complex functions as the depth increases, but this increased capacity does not lead to consistent in-between uncertainty across random restarts. Additionally, deep NLM are more susceptible to overfitting.}
    \label{fig:rr_depth}
\end{figure*}

\begin{figure*}
    \centering
    \includegraphics[width=\textwidth]{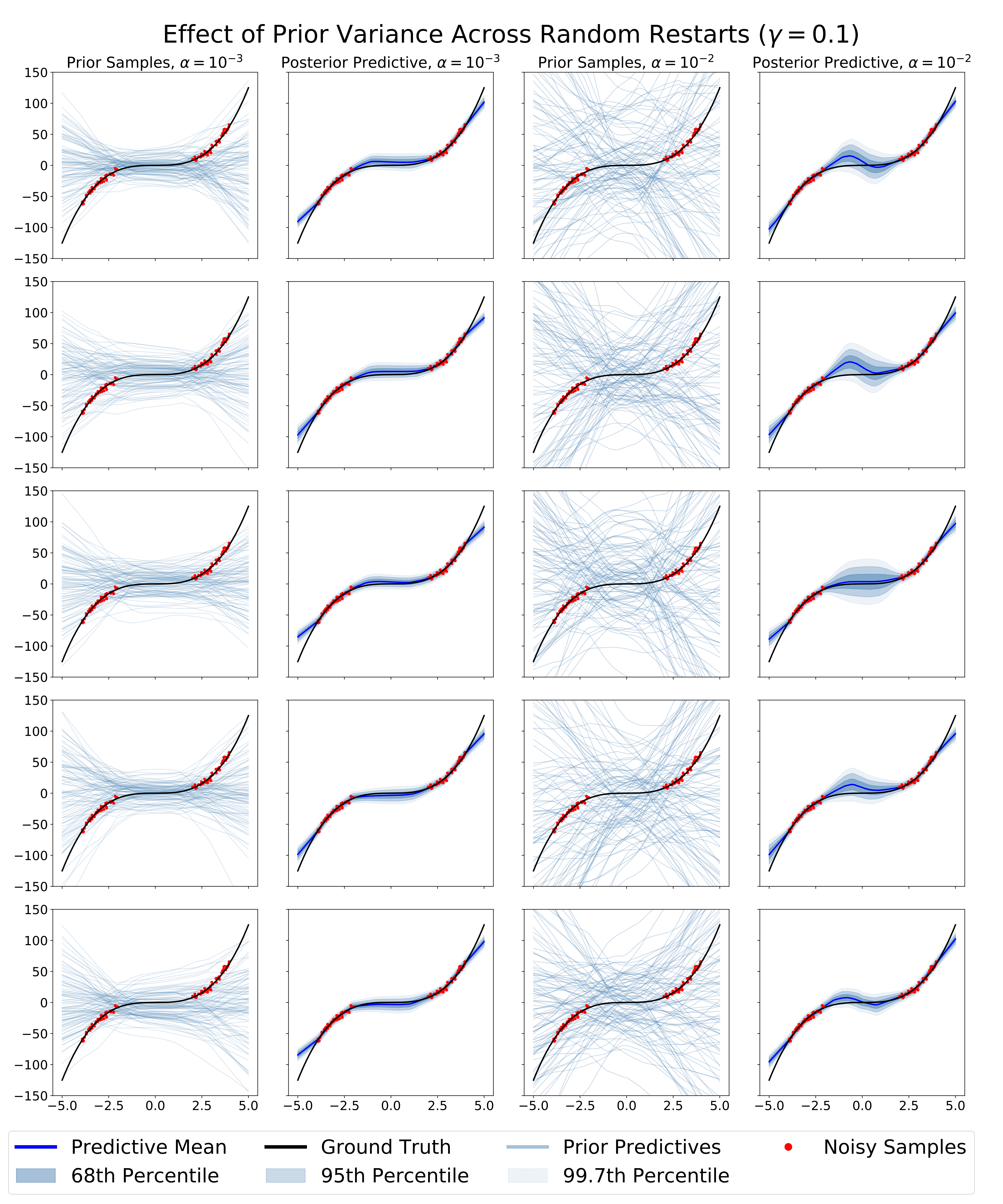}
    \caption{\textbf{Increasing the variance of the prior over the weights does not help capture in-between uncertainty consistently ($\gamma=0.1$).} We used marginal likelihood training and $\gamma=0.1$ to train an NLM with a 2-layer ReLU network with 50 and 20 neurons in the first and second layers respectively (i.e. 20 features). We see that the NLM is able to capture higher in-between uncertainty when $\alpha$ is high enough, but is inconsistent in doing so across random restarts.}
    \label{fig:marg_lowgamma}
\end{figure*}

\begin{figure*}
    \centering
    \includegraphics[width=\textwidth]{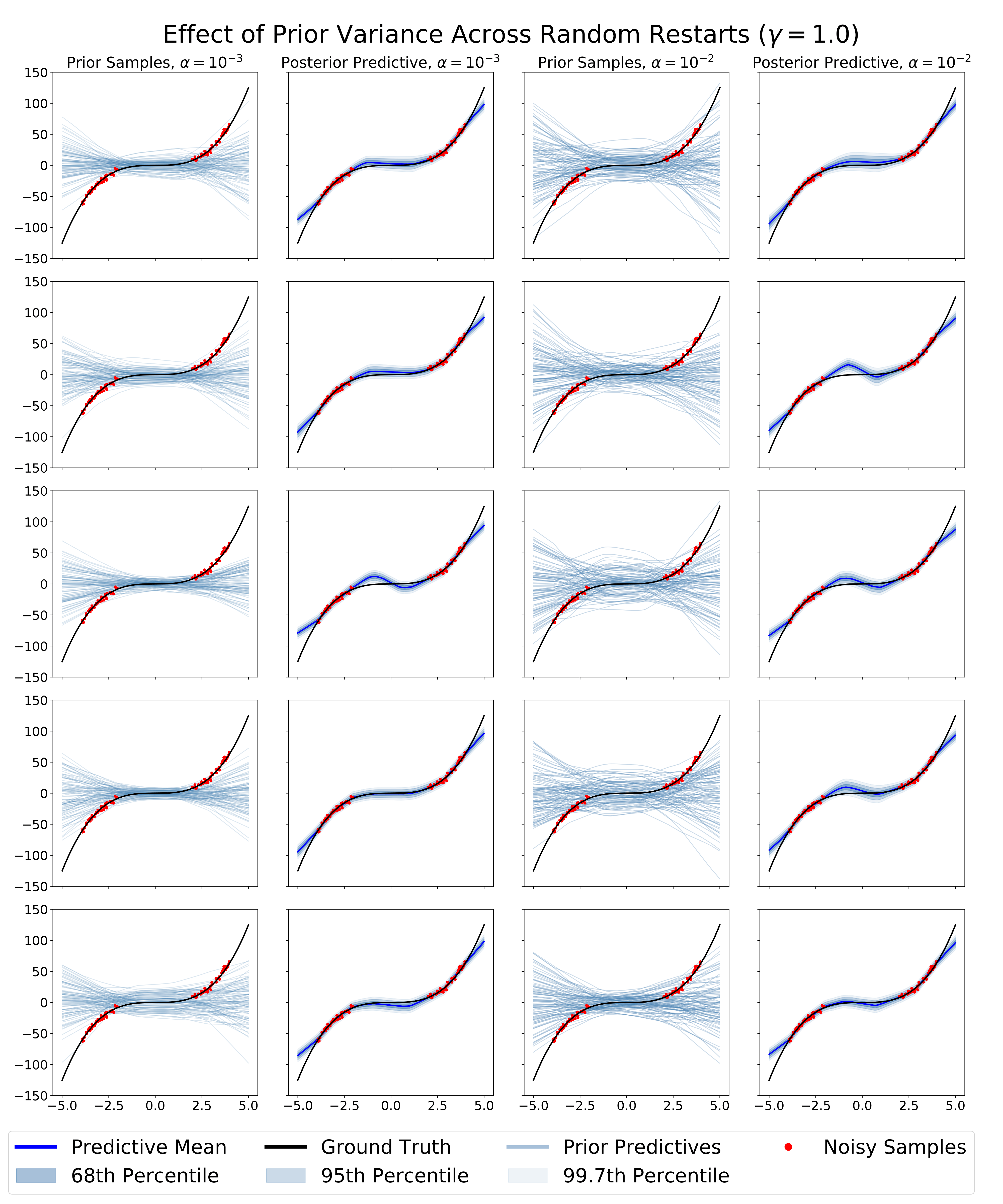}
    \caption{\textbf{Increasing the variance of the prior over the weights does not help capture in-between uncertainty consistently ($\gamma=1.0$).} We used marginal likelihood training and $\gamma=1.0$ to train an NLM with a 2-layer ReLU network with 50 and 20 neurons in the first and second layers respectively (i.e. 20 features). We see that the NLM is able to capture higher in-between uncertainty when $\alpha$ is high enough, but is inconsistent in doing so across random restarts.}
    \label{fig:marg_highgamma}
\end{figure*}

\begin{figure*}
    \centering
    \includegraphics[width=\textwidth]{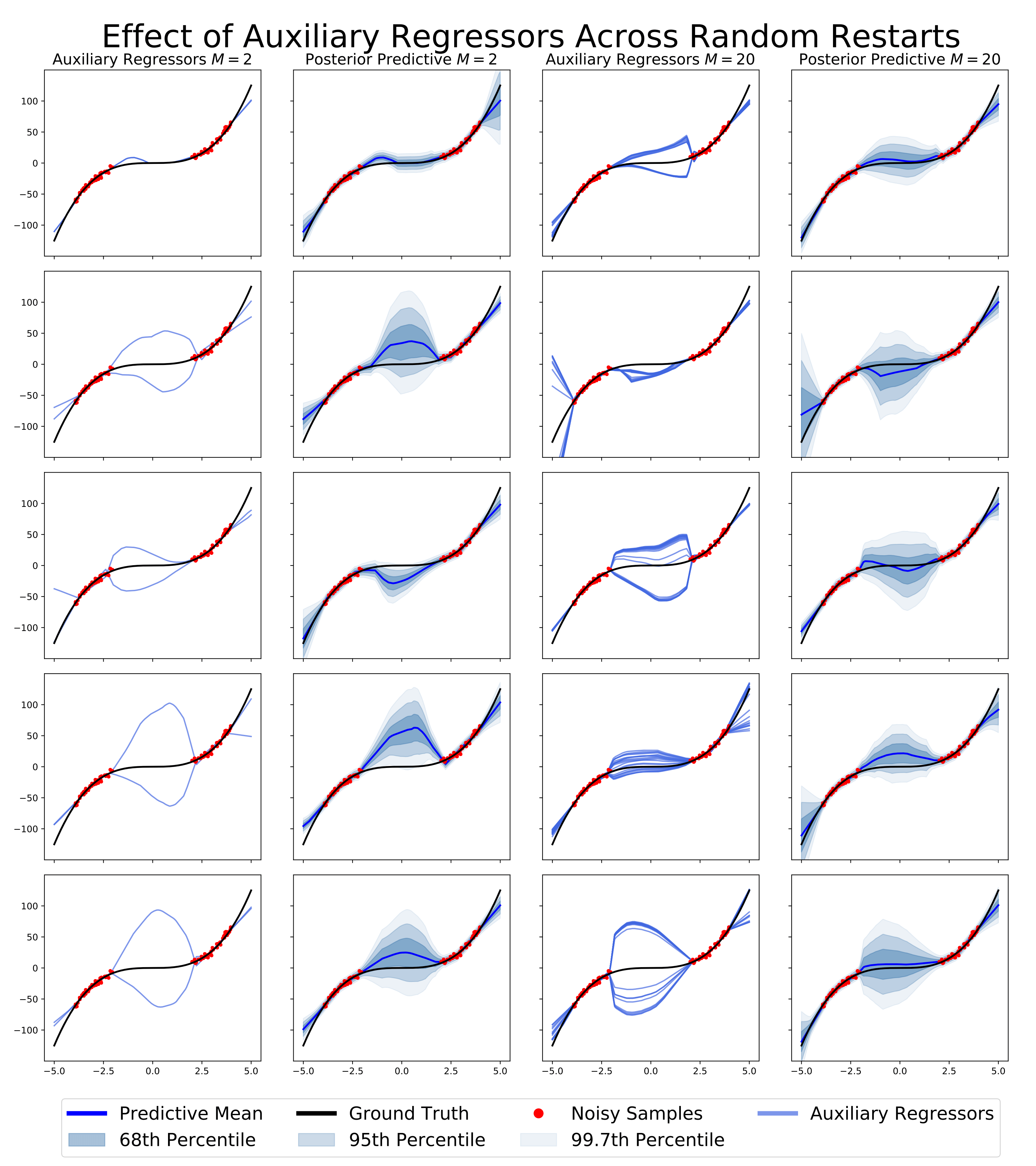}
    \caption{\textbf{LUNA consistently captures in-between uncertainty.} We used LUNA to train an NLM with a 2-layer ReLU network with 50 and 20 neurons in the first and second layers respectively (i.e. 20 features). We see that LUNA encourages for functional diversity under the prior predictive, and as a result is able to capture in-between uncertainty consistently across random restarts.}
    \label{fig:aux_rr}
\end{figure*}

\subsection{Additional Quantitative Results} \label{sec:additional-quantitative-results}

Tables \ref{tab:uci_regression_ll} and Tables \ref{tab:uci_regression_rmse} compare LUNA's log-likelihood and RMSE on the standard (non-gap) UCI data. As the tables show, LUNA has comparably good fit on these data-sets relative to baselines.

\begin{table*}
    \centering
    \text{Avg. Log-Likelihood (Test)}
    \resizebox{\textwidth}{!}{
    \begin{tabular}{ccccccc}
        \toprule
        & Boston & Concrete & Yacht & Kin8nm & Energy & Wine \\ 
		\cmidrule(lr){2-2} \cmidrule(lr){3-3} \cmidrule(lr){4-4} \cmidrule(lr){5-5} \cmidrule(lr){6-6} \cmidrule(lr){7-7}
		ENS BOOT & N/A & N/A & N/A & N/A & N/A & N/A \\
		ENS VAN & N/A & N/A & N/A & N/A & N/A & N/A \\
		ENS ANC & N/A & N/A & N/A & N/A & N/A & N/A \\
		\cmidrule{1-7}
		NLM & -3.67 $\pm$ 0.01 & -5.33 $\pm$ 0.00 & -2.32 $\pm$ 0.01 & 1.03 $\pm$ 0.03 & -2.78 $\pm$ 0.00 & -1.02 $\pm$ 0.03 \\
		GP & -3.69 $\pm$ 0.02 & -5.34 $\pm$ 0.00 & -2.69 $\pm$ 0.13 & 0.91 $\pm$ 0.03 & -2.86 $\pm$ 0.01 & -1.03 $\pm$ 0.04 \\
		MCD & -3.67 $\pm$ 0.01 & -5.32 $\pm$ 0.00 & -2.31 $\pm$ 0.01 & 0.93 $\pm$ 0.04 & -2.77 $\pm$ 0.00 & -1.02 $\pm$ 0.04 \\
		SNGP & -3.66 $\pm$ 0.01 & -5.33 $\pm$ 0.00 & -2.34 $\pm$ 0.03  & 0.86 $\pm$ 0.15 & -2.79 $\pm$ 0.02 & -1.05 $\pm$ 0.05 \\
		BBVI & -3.77 $\pm$ 0.02 & -5.37 $\pm$ 0.00 & -2.61 $\pm$ 0.07 & 0.97 $\pm$ 0.02 & -2.90 $\pm$ 0.01 & -1.08 $\pm$ 0.04 \\
		\cmidrule{1-7}
		LUNA & -3.67 $\pm$ 0.01 & -5.33 $\pm$ 0.00 & -2.31 $\pm$ 0.01 & 1.02 $\pm$ 0.03 & -2.79 $\pm$ 0.00 & -1.05 $\pm$ 0.06 \\
	\end{tabular}}
	\caption{\textbf{LUNA has comparable log-likelihood with baselines on standard (non-gap) UCI data.}}
	\label{tab:uci_regression_ll}
\end{table*}

\begin{table*}
    \centering
    \text{Root Mean Square Error (Test)}
    \resizebox{\textwidth}{!}{
    \begin{tabular}{ccccccc}
        \toprule
        & Boston & Concrete & Yacht & Kin8nm & Energy & Wine \\ 
		\cmidrule(lr){2-2} \cmidrule(lr){3-3} \cmidrule(lr){4-4} \cmidrule(lr){5-5} \cmidrule(lr){6-6} \cmidrule(lr){7-7}
		ENS BOOT & 2.86 $\pm$ 0.97 & 4.68 $\pm$ 0.50 & 0.81 $\pm$ 0.38 & 0.08 $\pm$ 0.00 & 0.50 $\pm$ 0.07 & 0.57 $\pm$ 0.06 \\
		ENS VAN & 2.78 $\pm$ 0.91 & 4.48 $\pm$ 0.57 & 0.53 $\pm$ 0.24 & 0.08 $\pm$ 0.00 & 0.42 $\pm$ 0.06 & 0.58 $\pm$ 0.07 \\
		ENS ANC & 2.80 $\pm$ 0.80 & 4.61 $\pm$ 0.52 & 0.67 $\pm$ 0.23 & 0.08 $\pm$ 0.00 & 0.53 $\pm$ 0.06 & 0.58 $\pm$ 0.06 \\
		\cmidrule{1-7}
		NLM & 3.11 $\pm$ 0.93 & 4.68 $\pm$ 0.65 & 0.55 $\pm$ 0.30 & 0.08 $\pm$ 0.00 & 0.37 $\pm$ 0.06 & 0.59 $\pm$ 0.04 \\
		GP & 4.21 $\pm$ 1.22 & 11.94 $\pm$ 0.40 & 3.45 $\pm$ 0.90 & 0.09 $\pm$ 0.01 & 2.58 $\pm$ 0.23 & 0.60 $\pm$ 0.05 \\
		MCD & 2.94 $\pm$ 0.86 & 4.36 $\pm$ 0.68 & 0.58 $\pm$ 0.22 & 0.09 $\pm$ 0.01 & 0.40 $\pm$ 0.08 & 0.58 $\pm$ 0.05 \\
		SNGP & 3.06 $\pm$ 0.93 & 5.00 $\pm$ 0.50 & 1.10 $\pm$ 0.43  & 0.10 $\pm$ 0.01 & 0.87 $\pm$ 0.50 & 0.62 $\pm$ 0.07 \\
		BBVI & 5.31 $\pm$ 1.29 & 16.45 $\pm$ 0.62 & 1.91 $\pm$ 0.51 & 0.09 $\pm$ 0.00 & 2.39 $\pm$ 0.16 & 0.64 $\pm$ 0.05 \\
		\cmidrule{1-7}
		LUNA & 3.18 $\pm$ 1.00 & 4.70 $\pm$ 0.56 & 0.51 $\pm$ 0.20 & 0.08 $\pm$ 0.00 & 0.40 $\pm$ 0.06 & 0.62 $\pm$ 0.08 \\
	\end{tabular}}
	\caption{\textbf{LUNA has comparable RMSE with baselines on standard (non-gap) UCI data.}}
	\label{tab:uci_regression_rmse}
\end{table*}

\begin{table*}
    \centering
    \text{Avg. Epistemic Uncertainty}
    \resizebox{\textwidth}{!}{
    \begin{tabular}{ccccccccccccc}
        \toprule
        & \multicolumn{2}{c}{Yacht - FROUDE} & \multicolumn{2}{c}{Concrete - CEMENT} 
        & \multicolumn{2}{c}{Concrete - SUPER} & \multicolumn{2}{c}{Boston - RM}
        & \multicolumn{2}{c}{Boston - LSTAT} & \multicolumn{2}{c}{Boston - PTRATIO} \\
		\cmidrule(lr){2-3} \cmidrule(lr){4-5} \cmidrule(lr){6-7} \cmidrule(lr){8-9} \cmidrule(lr){10-11} \cmidrule(lr){12-13}
		& Not Gap & Gap & Not Gap & Gap & Not Gap & Gap & Not Gap & Gap & Not Gap & Gap & Not Gap & Gap \\
		\cmidrule{2-13}
		ENS BOOT & 0.73 $\pm$ 0.11 & 0.58 $\pm$ 0.03 & 3.63 $\pm$ 0.24 & 5.10 $\pm$ 0.15 & 4.46 $\pm$ 0.39 & 8.32 $\pm$ 0.23 & 1.54 $\pm$ 0.18 & 1.38 $\pm$ 0.04 & 1.59 $\pm$ 0.20 & 1.67 $\pm$ 0.08 & 1.61 $\pm$ 0.17 & 1.58 $\pm$ 0.08 \\
		ENS VAN & 0.39 $\pm$ 0.03 & 0.49 $\pm$ 0.02 & 2.10 $\pm$ 0.23 & 4.58 $\pm$ 0.21 & 1.92 $\pm$ 0.20 & 5.01 $\pm$ 0.15 & 0.92 $\pm$ 0.12 & 0.85 $\pm$ 0.02 & 0.92 $\pm$ 0.10 & 1.09 $\pm$ 0.03 & 0.90 $\pm$ 0.14 & 1.23 $\pm$ 0.04 \\
		ENS ANC & 0.60 $\pm$ 0.23 & 0.63 $\pm$ 0.19 & 2.76 $\pm$ 0.38 & 5.25 $\pm$ 0.29 & 2.53 $\pm$ 0.36 & 6.23 $\pm$ 0.49 & 1.18 $\pm$ 0.16 & 1.06 $\pm$ 0.05 & 1.10 $\pm$ 0.10 & 1.20 $\pm$ 0.08 & 1.13 $\pm$ 0.19 & 1.44 $\pm$ 0.07 \\
		\cmidrule{1-13}
		NLM & 0.12 $\pm$ 0.01 & 0.15 $\pm$ 0.02 & 0.77 $\pm$ 0.05 & 0.73 $\pm$ 0.04 & 0.74 $\pm$ 0.05 & 0.78 $\pm$ 0.04 & 0.37 $\pm$ 0.03 & 0.34 $\pm$ 0.03 & 0.90 $\pm$ 0.07 & 0.85 $\pm$ 0.06 & 0.76 $\pm$ 0.07 & 0.73 $\pm$ 0.04 \\
		GP & 0.94 $\pm$ 0.13 & 1.59 $\pm$ 0.10 & 3.04 $\pm$ 0.25 & 5.31 $\pm$ 0.07 & 2.76 $\pm$ 0.27 & 6.01 $\pm$ 0.13 & 1.88 $\pm$ 0.14 & 1.57 $\pm$ 0.02 & 1.76 $\pm$ 0.15 & 1.71 $\pm$ 0.06 & 1.93 $\pm$ 0.25 & 2.23 $\pm$ 0.08 \\
		MCD & 1.61 $\pm$ 0.23 & 0.75 $\pm$ 0.08 & 1.36 $\pm$ 0.07 & 1.45 $\pm$ 0.03 & 1.29 $\pm$ 0.08 & 1.44 $\pm$ 0.03 & 0.79 $\pm$ 0.08 & 0.65 $\pm$ 0.02 & 0.80 $\pm$ 0.07 & 0.70 $\pm$ 0.04 & 0.81 $\pm$ 0.07 & 0.77 $\pm$ 0.03 \\
		SNGP & 0.07 $\pm$ 0.01 & 0.10 $\pm$ 0.05 & 0.33 $\pm$ 0.03 & 0.31 $\pm$ 0.03 & 0.28 $\pm$ 0.04 & 0.27 $\pm$ 0.04 & 0.29 $\pm$ 0.04 & 0.26 $\pm$ 0.04 & 0.56 $\pm$ 0.05 & 0.52 $\pm$ 0.06 & 0.29 $\pm$ 0.03 & 0.27 $\pm$ 0.02 \\
		BBVI & 1.60 $\pm$ 0.11 & 1.43 $\pm$ 0.02 & 2.79 $\pm$ 0.16 & 2.28 $\pm$ 0.09 & 6.34 $\pm$ 0.35 & 7.00 $\pm$ 0.70 & 2.48 $\pm$ 0.20 & 2.13 $\pm$ 0.05 & 2.50 $\pm$ 0.18 & 1.70 $\pm$ 0.08 & 5.70 $\pm$ 0.38 & 6.84 $\pm$ 0.45 \\
		\cmidrule{1-13}
        LUNA & 0.44 $\pm$ 0.10 & 0.68 $\pm$ 0.14 & 1.32 $\pm$ 0.15 & 2.05 $\pm$ 0.37 & 1.45 $\pm$ 0.31 & 7.29 $\pm$ 2.61 & 1.12 $\pm$ 0.09 & 1.00 $\pm$ 0.07 & 1.81 $\pm$ 0.25 & 2.30 $\pm$ 0.22 & 1.14 $\pm$ 0.09 & 1.83 $\pm$ 0.52 \\
	\end{tabular}}
	\caption{\textbf{On UCI ``gap'' data, LUNA captures higher in-gap epistemic uncertainty
where baselines struggle.} This table shows epistemic uncertainty computed both in the gap and outside of the gap. That is, it shows the raw epistemic uncertainty values used to compute the Epistemic Uncertainty Relative Change in Table \ref{tab:uci-gap-epistemic}.}
	\label{tab:uci_gap_epistemic_values}
\end{table*}

\end{document}